\documentclass{article}

\PassOptionsToPackage{numbers, compress}{natbib}

\usepackage[preprint]{neurips_2025}

\usepackage{ifthen}
\newboolean{includeappendix}
\setboolean{includeappendix}{true} % change to true to include appendix

\usepackage[utf8]{inputenc} % allow utf-8 input
\usepackage[T1]{fontenc}    % use 8-bit T1 fonts
\usepackage[colorlinks=true,linkcolor=blue,urlcolor=blue,citecolor=blue,anchorcolor=blue]{hyperref}       % hyperlinks
\usepackage{url}            % simple URL typesetting
\usepackage{booktabs}       % professional-quality tables
\usepackage{amsfonts}       % blackboard math symbols
\usepackage{nicefrac}       % compact symbols for 1/2, etc.
\usepackage{microtype}      % microtypography
\usepackage{xcolor}         % colors

\usepackage{bbm}
\usepackage{amsthm}
\usepackage{amsmath}
\usepackage{amssymb}
\usepackage{cleveref}
\crefname{theo}{theorem}{Theorem}
\usepackage{graphicx}
\usepackage{physics}

\usepackage{bibunits}             
\usepackage{wrapfig}
\usepackage{caption}
\usepackage{enumitem}
\makeatother

\newcommand{\sset}[1]{\left\{ #1 \right\}}
\newcommand{\virg}[1]{``#1''}
\newcommand{\ones}{\mathbbm{1}}
\newcommand{\Reals}{\mathbb{R}}
\newcommand{\Pro}{\mathbb{P}}

\newtheorem{proposition}{Proposition}
\newtheorem{theorem}{Theorem}

\newtheorem{lemma}{Lemma}
\newtheorem{definition}{Definition}

\title{Bound by semanticity: universal laws governing the generalization-identification tradeoff}

\author{%
\textbf{Marco Nurisso}\normalfont{\textsuperscript{1,2}},
\textbf{Jesseba Fernando}\normalfont{\textsuperscript{3,4}},
\textbf{Raj Deshpande}\normalfont{\textsuperscript{5}},
\textbf{Alan Perotti}\normalfont{\textsuperscript{2}},\\
\textbf{Raja Marjieh}\textsuperscript{6},
\textbf{Steven M. Frankland}\textsuperscript{7},
\textbf{Richard L. Lewis}\textsuperscript{8},
\textbf{Taylor W. Webb}\textsuperscript{9},\\
\textbf{Declan Campbell}\textsuperscript{10},
\textbf{Francesco Vaccarino}\textsuperscript{1},
\textbf{Jonathan D. Cohen}\textsuperscript{6,10},
\textbf{Giovanni Petri}\textsuperscript{2,5,11}\\
\textsuperscript{1}Dipartimento di Scienze Matematiche, Politecnico di Torino \\
\textsuperscript{2}CENTAI Institute\\
\textsuperscript{3}Network Science Institute, Northeastern University\\
\textsuperscript{4}Institute for Experiential AI, Northeastern University\\
\textsuperscript{5}NP Lab, Network Science Institute, Northeastern University London\\
\textsuperscript{6}Department of Psychology, Princeton University\\
\textsuperscript{7}Program in Cognitive Science, Dartmouth College\\
\textsuperscript{8}Department of Psychology, University of Michigan\\
\textsuperscript{9}Microsoft Research\\
\textsuperscript{10}Princeton Neuroscience Institute\\
\textsuperscript{11} Department of Physics, Northeastern University
}

\begin{document}

\maketitle
\begin{bibunit}[unsrtnat]
\begin{abstract}
Intelligent systems must deploy internal representations that are simultaneously structured—to support broad generalization—and selective—to preserve input identity. 
We expose a fundamental limit on this tradeoff. 
For any model whose representational similarity between inputs decays with finite semantic resolution $\epsilon$, we derive closed‑form expressions that pin its probability of correct generalization $p_S$ and identification $p_I$ to a universal Pareto front independent of input space geometry.  
Extending the analysis to noisy, heterogeneous spaces and to $n>2$ inputs predicts a sharp $1/n$ collapse of multi-input processing capacity and a non‑monotonic optimum for $p_S$.  
A minimal ReLU network trained end‑to‑end reproduces these laws: during learning a resolution boundary self‑organizes and empirical $(p_S,p_I)$ trajectories closely follow theoretical curves for linearly decaying similarity. 
Finally, we demonstrate that the same limits persist in two markedly more complex settings—a convolutional neural network and state‑of‑the‑art vision–language models—confirming that finite‑resolution similarity is a fundamental emergent informational constraint, not merely a toy‑model artifact.  
Together, these results provide an exact theory of the generalization‑identification trade‑off and clarify how semantic resolution shapes the representational capacity of deep networks and brains alike.
\end{abstract}

\section{Introduction}
\paragraph*{Background.}
Modern neural networks have revolutionized the way we perform everyday tasks. 
However, they exhibit inherent tradeoffs between information processing and generalization observed in cognitive systems \citep{campbell2024understanding}. 
These networks utilize distributed representations that enable efficient generalization in unseen situations \citep{Hinton,Hinton86, smolensky1990tensor}, but suffer from the binding problem ---the inability to maintain associations between features when processing multiple inputs simultaneously \citep{Roskies99,Greff20,treisman1980feature}. 

To understand how internal representations help to generalize, we build on Shepard's Universal Law of Generalization \citep{Shepard58, shepard1987toward}, which frames generalization as minimizing distances in psychological space--a principle supported by studies using similarity judgments in perceptual spaces \citep{Xie24, Schurgin20, Tomic24}.
The law states that generalization is the minimization of the distance between a new stimulus and internal representations in a psychological space, which has been further supported in a wide range of literature \citep{zaslavsky2018efficient,cheng2000honeybees,hebart2020multidim}. This principle was formalized through information bottleneck theory \citep{Tenenbaum2001Generalization, Sims18, tishby2015deep}, where neural systems optimize the tradeoff between compression and task performance \citep{Sims16, Tishby99, Schwartz-Ziv17}.

These theoretical limitations manifest in modern neural architectures: \citet{Gong24} demonstrated bottlenecks in attention heads analogous to working memory limits, and \citet{campbell2024understanding} showed that vision-language models fail at multi-item processing due to representation interference. 
Most notably, \citep{frankland2021no} proposed that the reason these
limits emerge is the same reason for their success: because they build structured representations that are suited for generalization. 
% Formally, \citet{Alshammari25} developed the Information Contrastive Learning (I-CON) framework, revealing similarity structures in learned representations that mirror Shepard's exponential functions.

\paragraph{Our contribution. }
We investigate the fundamental tradeoff between representational fidelity and distinctness under finite semantic resolution. 
More precisely, we provide:

\begin{enumerate}[leftmargin=10pt]
    \item A framework that quantifies the exact Pareto front between identification and similarity performances, demonstrating how finite resolution creates an inescapable tradeoff;
    \item Closed-form expressions for this tradeoff across multiple inputs, noise levels, and varying resolutions, revealing a sharp $1/n$ collapse in multi-item ($n$) processing capacity;
    \item Empirical validation showing how this resolution boundary self-organizes during neural network training, with empirical trajectories closely following our theoretical predictions;
    \item Confirmation that these limits persist across architectures from simple ReLU networks, to CNNs, to vision-language models, establishing finite resolution as a universal constraint rather than a model-specific artifact. 
\end{enumerate}

\section{Setup}\label{section:setup}

\paragraph{Stimulus space and similarity functions.}
Assume $A$ to be a model processing stimuli coming from a set $S$ the structure of which is encoded by a distance function $d_S$.
For example, $S$ can be the space of color hues or days of the week, naturally arranged in a circle, the set of positions of an item in physical space, or more complex topological spaces, such as a torus, or the Klein bottle of natural image patches \citep{carlsson2008local}.

The model processes the stimuli coming from $S$ and builds representations by mapping them into a latent (or psychological) space $M$ with a map $\Phi:S\to M$, which we assume to be a bijection: this induces naturally a distance $d$ on $M$ via $d(x,y):=d_S(\Phi^{-1}(x),\Phi^{-1}(x))$.
In $M$, the representations are processed and compared through a non-negative similarity function $g\colon M\times M \to \Reals_+$.
For example, if $M$ is a vector space, we can choose $g(x,y)= h(\Phi(x)^\top \Phi(y))$ with $h(x)\geq 0\, \forall x$.
If $h(x) =\exp(-x)$, this encompasses, but is more general than, the standard self-attention mechanism of a transformer \citep{vaswani2017attention} 
\footnote{Our similarity function includes common ML metrics: cosine similarity in embedding models, dot-product attention in transformers, and implicit similarity in contrastive learning (InfoNCE, triplet loss). 
While these mechanisms differ in implementation, they all measure semantic relatedness between representations and are subject to the resolution limits we identify in this work.}.

The specific form of $g$ is not uniquely specified by the distance $d$, allowing for different degrees of \virg{semanticity}--how the metrical structure $d$ is represented by $g$-- with significant impacts on model capabilities. 
Localized functions $g_x := g(x,\cdot)$ reduce interference between representations, permitting more reliable distinction between them and thus accurate simultaneous processing of multiple representations.
Conversely, more distributed $g$ can reflect long-range relations of $S$, thus enhancing generalization capabilities, at the cost of potential interference among distinct but nearby stimuli.
In the following, corroborated by seminal works in the cognitive psychology literature \citep{shepard1987toward}, we assume for simplicity that $g$ depends only on the distance between the stimuli: $g(x,y) = g(d(x,y))$. 
%
% This choice is corroborated by seminal works in the cognitive psychology literature \citep{shepard1987toward}, where evidence is provided for the presence of an exponential \virg{generalization gradient} (akin to a similarity function) $g(x,y)=e^{-\mu d(x,y)}$ in human similarity judgments. 
% \citet{sims2018efficient}, moreover, proves the optimality of this form in the setting of rate-distortion theory.

% Notice the presence of a parameter $\mu>0$ in Shepard's exponential, acting as the similarity's sensibility to distance: the smaller $\mu$, the larger the similarity between far-away stimuli.
%
\begin{figure}
    \centering
    \includegraphics[width=0.92\linewidth]{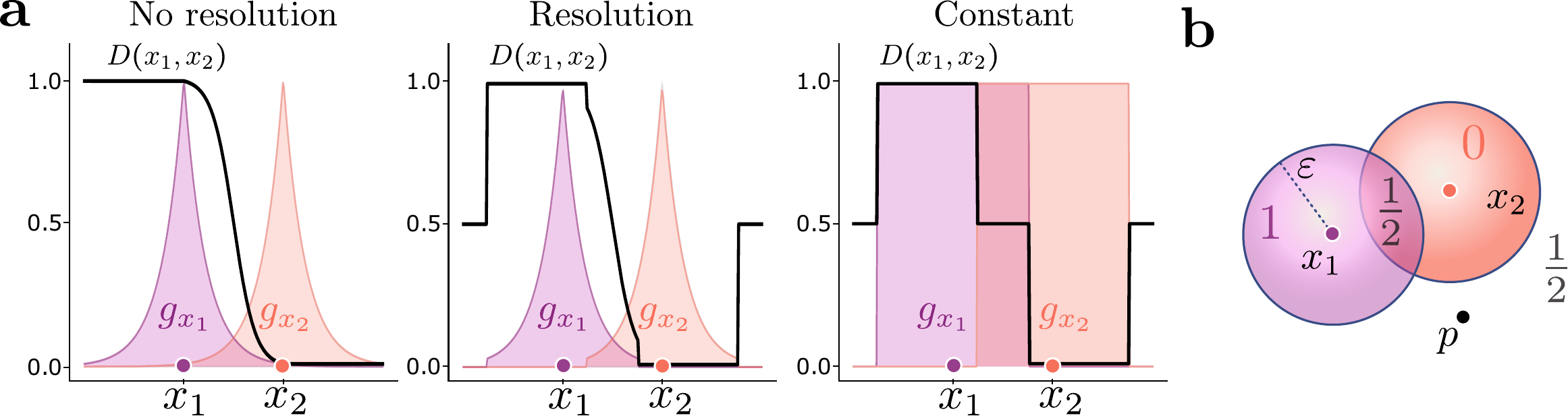}
    \caption{
    \textbf{a.} On the left, exponential similarity functions centered on two stimuli $x_1,x_2\in M$, with the black line indicating the decision function $g(x_1,p)/(g(x_1,p)+g(x_2,p))$ with no resolution (see \Cref{section:setup} for explanation). On the center and right, the same quantities are shown in the case of the presence of finite resolution. Notice that the model becomes uncertain for probes far away from stimuli $x_1,x_2$. 
    \textbf{b.} Visualization of the constant similarity functions of \Cref{def:constant_sim}.}
    \label{fig:similarity_functions}
\end{figure}
\paragraph{Measures of identification and generalization accuracy.}
Following \citet{frankland2021no}, we introduce models of two simple tasks that have previously been used to measure identification and generalization accuracy, and that we use in our theoretical analyses below.

We measure the generalization capabilities of $A$ using a \emph{similarity task} in which the model is asked to perform similarity judgments that respect the metric structure of the stimulus space.
The model is shown $n$ stimuli $x_1,\dots, x_n\in S$ and additional one, called the \emph{probe}, $p\in S$. 
It is then asked to decide which of the $n$ stimuli is the closest to $p$ according to the distance $d$.
Let $(x_1,\dots,x_n),p$ be sampled independently from $M$ according to a probability measure $\nu$.
We call $X$ the random variable encoding the index of the closest item to the probe, i.e. $X = \underset{i=1,\dots,n}{\mathrm{argmin}}\, d(x_i,p)$.
Intuitively, the decision function represents how the model assesses the evidence when determining which input is most similar to the probe.
It formalizes the idea that the model's choice depends on relative similarity strengths rather than absolute values.
We call $Y$ the random variable indicating the model's decision, that we model as follows \citep{luce1959individual}:
\begin{equation}\label{eq:decision_sim}
    D_i(x_1,\dots,x_n;p) 
 :=  \mathbb{P}(Y=i|(x_1,\dots,x_n,p)) = \frac{g(x_i,p)}{\sum_{k=1}^n g(x_k,p)}.
\end{equation}
We quantify the overall generalization capability as the probability of the model making the correct decision, i.e. $p_S := \Pro(Y=X)$.

The \emph{identification} task is used to measure how accurately stimuli can distinguished from one another.
The task is the same as the similarity task, but with the exception that the probe is always one of the input stimuli $p\in\sset{x_1,\dots,x_n}$.
This will result in the decision function of \Cref{eq:decision_sim} always being of the form
\begin{equation}\label{eq:decision_id}
    D_i(x_1,\dots,x_n;x_j):=\mathbb{P}(Y=i|(x_1,\dots,x_n,x_j)) =  \frac{g(x_i,x_j)}{\sum_{k= 1}^n g(x_k,x_j)}.
\end{equation}
If now $X(x_1,\dots,x_n;x_j) = j$, we write $p_I:=\Pro(Y=X)$ to indicate the probability of the model succeeding in the identification task.
\Cref{eq:decision_sim,eq:decision_id} can be interpreted, independent of probabilities, in terms of relative similarity, where $p_S$ is taken to represent the average \emph{relative similarity} of stimuli that are close compared to stimuli that are further apart. 
In the same way, $p_I$ is the average relative similarity of equal stimuli compared to different stimuli.

% \mnote{Something quick about Shepard and Sims should probably go here.}
Importantly, when $g(x_i,x_j) = \exp(-\mu d(x_i,x_j))$, and the decay rate for the exponential is taken to infinity ($\mu\to\infty$), both $p_S$ and $p_I$ approach 1, (perfect performance); that is, identification and generalization accuracy both benefit by maximizing decay rate .
%In this work, we argue that this case is unrealistic, as its implementation would require the computation of similarities with infinite precision.
%We will see how these two quantities are in tension with one another when this computation is limited.
Critically, however, it has been observed empirically that virtually \textit{any} loss of precision (i.e., resolution) in computing the similarity function introduces a fundamental tension between $p_S$ (generalization) and 
$p_I$ (identification accuracy) with respect to decay rate, wherein generalization benefits by \textit{decreases} in decay rate that dramatically degrade identification accuracy (\Cref{fig:similarity_functions}a).  
This tension has been referred to as "Miller's Law" \citep{frankland2021no}.  Here, we provide a formal analysis of this effect, showing that it generalizes to learning in neural networks, where it imposes a fundamental constraint on the interaction between representations and efficiency of processing.

\paragraph{The effect of resolution.}
\footnote{Note on terminology: ``resolution" ($\varepsilon$) in this paper strictly refers to the parameter controlling the distance threshold beyond which similarities collapse to noise level $\Delta$. Higher $\varepsilon$ values mean the model preserves similarity information across greater distances.}
%To understand the effects of resolution, consider how humans differentiate between similar colors but struggle to rank hues.  Similarly, neural networks exhibit limited "granularity" in encoding relationships between distant points.
To show this, we formally consider how a  limit in precision with which the model can compute a similarity function 
%for stimuli which have large distances
impacts both identification and generalization accuracy. 
Such a limit might arise from any number of factors: computational noise, finite precision, ReLU activations clamping negative correlations to zero (see \Cref{section:nn}), or imprecisely coded distant relationships.
These can all be formalized as a \emph{resolution} $\varepsilon>0$ such that  $g(x,y) \approx \Delta$ if $d(x,y)>\varepsilon$, where $\Delta$ is a noise parameter.
As shown in \Cref{fig:similarity_functions}a, the resolution drastically affects decision boundaries (the \textbf{black} line): for probes sufficiently far from both stimuli, the decision function approaches $1/2$ indicating maximal uncertain.
Resolution thus represents the model's inherent limitation in gauging low similarities between distant stimuli.

\paragraph{Generalization-Identification Tradeoff (Miller's Law).} 
%\emph{The presence of resolution induces a tradeoff between the model's similarity and identification capabilities.}
To analyze this, we use a simplified similarity function.  If $\ones_A$ is the indicator function over the set $A$, and $B_r(x)$ is the closed ball of center $x$ and radius $r$ over $M$, $B_r(x)=\sset{y\in M:d(x,y)\leq r}$, the similarity function can be defined as follows:
\begin{definition}\label{def:constant_sim}
The constant similarity function with resolution $\varepsilon$ and noise $\Delta$ is $g_{\varepsilon;\Delta}(x,y)=\ones_{B_\varepsilon(x)}(y) + \Delta \ones_{M\setminus B_\varepsilon(x)}(y)$.
\end{definition}

According to this function, the model will judge two things to be similar ($g_{\varepsilon;\Delta}(x,y)=1$) if and only if they are closer than a certain threshold $\varepsilon>0$.
Outside of this \virg{resolution region} the similarity value is fixed to a noise value $\Delta>0$.

% \mnote{Missing discussion about Shepard}
% \jnote{What do you think of this for the connection?}\mnote{Great!}
This simplified model aligns with Shepard's Universal Law of Generalization \citep{shepard1987toward}, where similarity decays exponentially with distance: $g(x,y)=\exp(-\mu d(x,y))$. 
In Shepard's formulations, the parameter $\mu$ controls the sensitivity to distance, with larger $\mu$ creating sharper similarity boundaries. 
This is conceptually similar to controlling the temperature parameter in a softmax function, in which lower temperatures induce sharper probability distributions, while higher temperatures make them more uniform. 
In our framework, $\varepsilon$ serves an analogous role, controlling the distance of the similarity functions or the spatial range of entanglement (or \emph{semanticity}) of the representations.

Below, we use this to quantify the generalization-identification tradeoff as a function of $\varepsilon$.
\section{Theoretical results}
We use the constant similarity function defined above to derive closed form solutions for the values of $p_S$ and $p_I$ over a broad class of stimulus spaces and probability distributions over them.  

Accordingly, we denote $b_p(\varepsilon)$ as the probability measure of the closed ball of radius $\varepsilon$ centered in $p$, 
$b_p(\varepsilon) := \nu(B_\varepsilon(p))$. 
Furthermore, let $\langle b(\varepsilon)\rangle = \mathbb{E}_{p\sim\nu}[b_p(\varepsilon)]$ be the average measure of a ball of radius $\varepsilon$ in $M$, and $\mathrm{Var}(b(\varepsilon))$ its variance.
The variance term $\mathrm{Var}(b(\varepsilon))$ captures how the probability mass of $\varepsilon$-balls varies across space.
Intuitively, this measures the heterogeneity of the stimulus space--that is, how differently `crowded' regions are, which, in turn, compromises similarity judgments.
Additional assumptions and notations are described in \Cref{section:techinical}.

\begin{theorem}[$2$-item tests]\label{theo:2-item}
Let $(M,d,\Sigma,\nu)$ be a separable metric probability space.  
If, for every $p\in M$, $b_p$ is absolutely continuous on every closed sub-interval of $[0,\infty)$, then, for the noise-free constant similarity function $g = g_{\varepsilon;0}$ it holds that
\begin{align}
p_S(\varepsilon) &= \frac{1}{2} + \langle b(\varepsilon) \rangle - \langle b(\varepsilon)\rangle^2 - \mathrm{Var}(b(\varepsilon)),\label{eq:sim_2}\\
p_I(\varepsilon) &= 1 - \frac{1}{2}\langle b(\varepsilon) \rangle.  \label{eq:ide_2}
\end{align}
\end{theorem}
\vspace{-5mm}
\begin{proof}
The proofs %of these two results %are technical and 
can be found in \Cref{appendix:proof_1}.
\end{proof}
These results have implications for neural architecture design and quantify how much identification performance must be sacrificed to gain generalization ability.
These results, being independent of model choices, provide multiple insights on how $p_S,p_I$ depend on the resolution $\varepsilon$ and on their relation. 

First, note that the variance of the ball volume appears in \Cref{eq:sim_2} as a term responsible for decreasing the probability of success in the similarity test.
This happens when the probability distribution is non-uniform or the space is heterogeneous, as for a manifold with boundary). 
Spaces which are homogeneous (in Haar measure) with uniform probability distributions
% (e.g. homogeneous manifolds equipped with the Haar measure)
 will have $\mathrm{Var}(b(\varepsilon))=0$, hence performing similarity tests on them will be easier.
Therefore, models will perform better on uniform data manifolds (such as rotations), than on manifolds with varying density (such as natural images).

The specific values of $p_I(\varepsilon)$ and $p_S(\varepsilon)$ can vary depending on the space chosen. 
However, assuming $ \mathrm{Var}(b(\varepsilon)) = 0$, they are both parametrized by $\langle b(\varepsilon) \rangle$, which is always a non-decreasing function of $\varepsilon$ from 0 to 1.
This means that, in the $(p_S,p_I)$ plane, there is a \virg{universal} Pareto curve relating identification to generalization accuracy that is independent of $M$ and $\nu$ (\Cref{fig:curves_theorem}a).
Indeed, as we will show in \Cref{section:nn}, the distance of empirical performances from the Pareto front directly quantifies the additional `difficulty' introduced by the heterogeneity of the stimuli space (\Cref{fig:curves_theorem}b). 

\begin{wrapfigure}[35]{r}{0.4\textwidth}
  \begin{center}
    \includegraphics[width=0.45\textwidth]{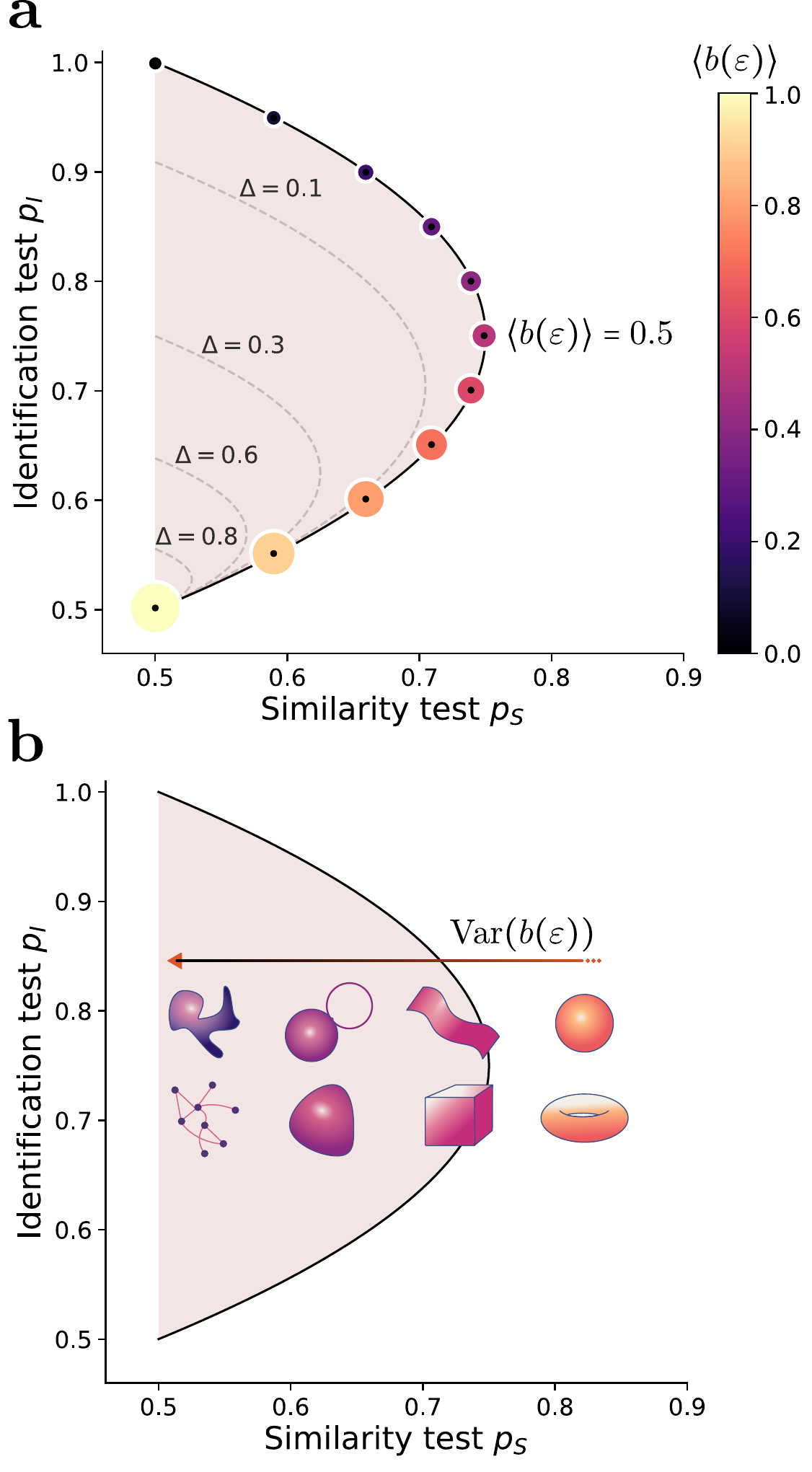}
  \end{center}
  \caption{\textbf{a.} The region in $(p_S,p_I)$ plane where the model's performances lie (\Cref{theo:2-item}). 
  The black line is parameterized by the resolution $\varepsilon$ and represents the behaviour of the model in homogeneous spaces. 
  \textbf{b.} Effect of heterogeneity $\mathrm{Var}(b(\varepsilon))$ on the similarity test performance.}
  \label{fig:curves_theorem}
\end{wrapfigure}
This curve exhibits three regimes as a function of the ball's resolution $\varepsilon$.

\textcolor[HTML]{3B0F70}{\textbf{Low} $\varepsilon$}\textbf{ regime.}
For small resolutions, the similarity functions act like Dirac deltas, meaning that representations do not interfere with one another and thus are perfectly distinguishable ($p_I \approx 1$).
However, small resolutions mean that the model is able to recognize two objects as similar only if they are very close, limiting generalization ($p_S \approx 0.5$, chance level).

\textcolor[HTML]{B73779}{\textbf{Medium} $\varepsilon$}\textbf{ regime.}
Increasing $\varepsilon$ elicits the similarity-identification tradeoff:  
%the model begins to encode the structure of the space, as stimuli \virg{reach} ones that are further away.
As $\varepsilon$ increases, the similarity measure for more distant stimuli becomes more robust, and thus the structure of the space can be more accurately represented. 
However, this comes at the cost of nearby stimuli becoming more similar, thereby producing interference that decreases $p_I$.
%, as stimuli tend to interfere more often with each other.
Importantly, $p_S$ reaches a maximum at $\langle b(\varepsilon)\rangle = \frac{1}{2}$, i.e. when the average ball covers half of the space.

\textcolor[HTML]{F76F5C}{\textbf{High} $\varepsilon$ }\textbf{ regime.}
Once $\epsilon$ increases beyond $\langle b(\varepsilon)\rangle > \frac{1}{2}$, the cases in which  stimuli interfere $d(x_1,p)\leq\varepsilon, d(x_2,p)\leq \varepsilon$ outweigh the ones in which the probe is too far away $d(x_1,p) > \varepsilon, d(x_2,p) > \varepsilon$, resulting in a decrease in both $p_S$ and $p_I$.

\vspace{-10pt}

\paragraph{The effect of noise.} The result of \Cref{theo:2-item} can be readily extended to take into account the presence of nonzero noise outside the resolution region.
\begin{theorem}[Noise]\label{theo:noise}
Under the same assumptions of \Cref{theo:2-item}, for the two-item similarity and identification tests with constant similarity functions $g=g_{\varepsilon;\Delta} $ with noise level $\Delta\geq 0$ it holds that
\begin{align}
p_S(\varepsilon,\Delta) &= \frac{1}{2} + \frac{1-\Delta}{1+\Delta}(\langle b(\varepsilon)\rangle - \langle b(\varepsilon)^2 \rangle),\\
p_I(\varepsilon,\Delta) &= \frac{2-(1-\Delta)\langle b(\varepsilon)\rangle}{2+2\Delta}.
\end{align}
\end{theorem}
\vspace{-5mm}
\begin{proof}
The proof can be found in \Cref{proof:noise}.
\end{proof}
The effect of noise can be appreciated in \Cref{fig:curves_theorem}a as a monotonous decrease in both $p_S$ and $p_I$.

\paragraph{Processing of multiple stimuli.}
The foregoing analyses may provide a formal account of why humans and large neural networks alike exhibit dramatic processing constraints in simple tasks (e.g. visual working memory tasks and numerosity judgments), that demand simultaneous processing of multiple stimuli \citep{campbell2024understanding}. 
On the one hand, these tasks typically demand generalization (e.g., the processing of stimuli that involve arbitrary combinations of features, such as color, shape and position). 
On the other hand, performance is typically evaluated based on identification accuracy by identifying individual stimuli. 
The results above thus suggest that these competing demands run up against the fundamental tension between identification and generalization accuracy,  irrespectively of scale or architecture (i.e., even in systems with billions of parameters, such as VLMs or the human brain). 
When such systems intrinsically value and/or are trained explicitly for generalization, then they will 
%According to our theoretical setup, we predict that they will do so by trading some identification capabilities and 
position themselves into the \textcolor[HTML]{3B0F70}{low}-\textcolor[HTML]{B73779}{medium} resolution/semanticity regime ( \Cref{fig:curves_theorem}a).
% \Cref{fig:curves_theorem}b show how this %heterogeneity%
% should affect performance. 
Indeed, we can show this is the case by explicitly deriving probabilities of success for $n$-item similarity and identification tasks.
\begin{theorem}[$n$-item tests]\label{theo:n-item}
Under the same assumptions of \Cref{theo:2-item}, for the constant noise-free ($\Delta=0$) similarity function $g =g_{\varepsilon;0}$ we have that
\begin{align}
p_{S}^n(\varepsilon) &= \mathbb{E}_{p\sim \nu}\left[\frac{1}{n} + \sum_{k=1}^{n-1} \frac{(1-b_p(\varepsilon))^{n-k} - (1-b_p(\varepsilon))^n}{k}\right], \label{eq:n_item_S}\\
p_{I}^n(\varepsilon) &= \mathbb{E}_{p\sim \nu}\left[\frac{1-(1-b_p(\varepsilon))^n}{nb_p(\varepsilon)}\right].\label{eq:n_item_I}
\end{align}
\end{theorem}
\begin{proof}
The proof can be found in \Cref{appendix:proof_3}.
\end{proof}

\begin{figure}
    \centering
    \includegraphics[width=\linewidth]{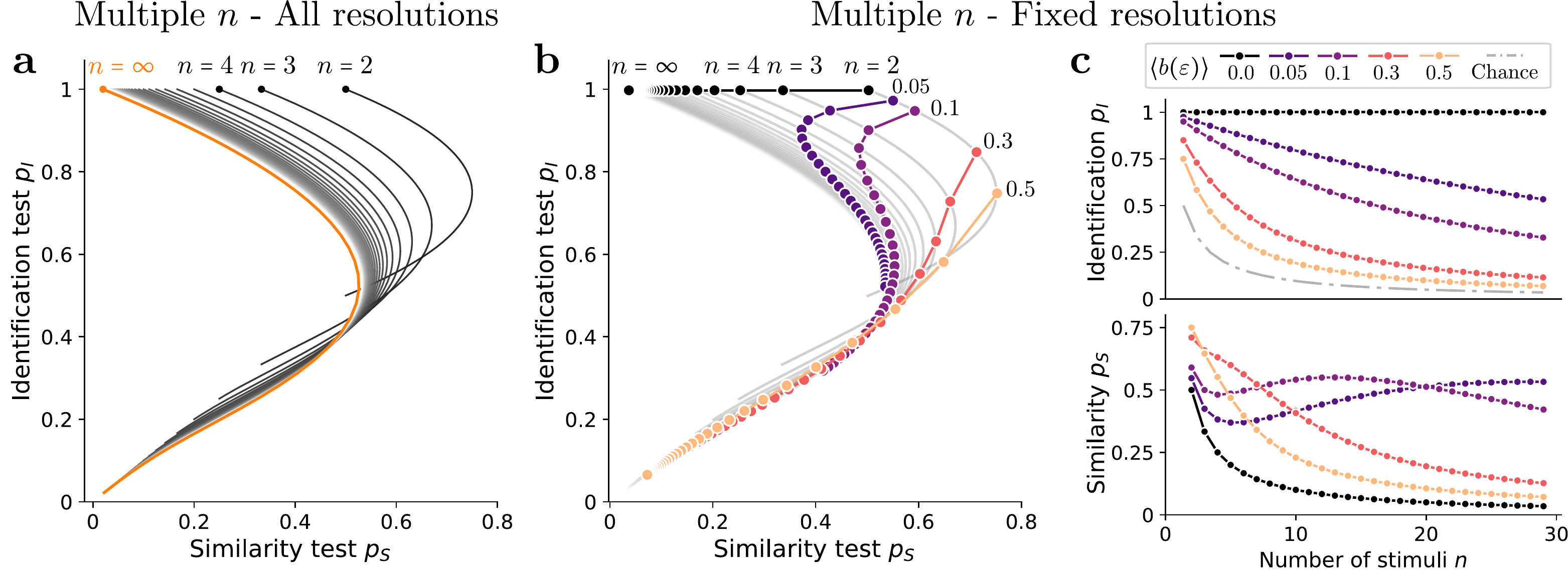}
    \caption{{\bf a.} Similarity-identification curves for different values of $n$ and parameterized  by $b_p(\varepsilon)\in [0,1]$, as described by \Cref{eq:n_item_S,eq:n_item_I}. {\bf b.} The colored curves correspond to similarity-identification values as the number of inputs $n$ varies, for some fixed values of $b_p(\varepsilon)$. {\bf c.} Similarity (top) and identification (bottom) dependence on $n$ for different resolutions.}
    \label{fig:n_items}
\end{figure}
First, note that, despite their apparently complicated formulations, \Cref{eq:n_item_S,eq:n_item_I} are polynomials in $b_p(\varepsilon)$ for any fixed $n$ and, given their non-linearity, the expected value over the probes cannot be simplified in general. 
Thus, for simplicity, we focus on the \emph{homogeneous} case where $b_p(\varepsilon) = b(\varepsilon)\ \forall p\in M$ and $\mathbb{E}$ disappears.

Under this assumption, both similarity and identification performances are once again parameterized by $b(\varepsilon)$, yielding universal pareto curves independent of $M$.
\Cref{fig:n_items}a shows the shape of the Pareto front for different values of $n$.
As a sanity check, note that, as the resolution goes to $b(\varepsilon)=0$, performance approaches perfect identification for any number of simultaneous inputs with no capacity to generalize $p_S^n(0) = 1/n$ (chance level).

As shown in \Cref{fig:n_items}(b,c), the mapping of one curve into the next is not \virg{uniform}. 
For any fixed  $\varepsilon>0$, increasing the number of inputs quickly degrades both identification and generalization performances. 
Furthermore, \Cref{eq:n_item_I} shows that for large $n$, $p_I^n(\varepsilon) \approx (b(\varepsilon)n)^{-1}$: identification performance decrease as $1/n$ with a rate given by $b(\varepsilon)$.
For a model tasked with learning structured representations of the input space, and thus optimizing for generalization (say, $ b(\varepsilon) \approx 1/2$ for $n=2$), our analyses predict that the capacity to accurately process multiple representations at the same time will be strongly constrained (\Cref{fig:n_items}c).

Interestingly, the bottom panel of \Cref{fig:n_items}c shows that the probability of success in the similarity test is non-monotonic in $n$ when $b(\varepsilon)$ is small.
Thus, when the model has to deal with a high number of items, it is convenient for it to pick low resolutions. 
The cost, however, is paid by the significant increase in error for low numbers of items. 

These observations provide an elegant explanation for why even large neural network models struggle with multi-object reasoning \cite{campbell2024understanding}: they likely have developed representations that support generalization, but this brings  a $1/n$ decrease in identification probability as the number $n$ of objects increase, thus generating the striking capacity limits observed in both humans and large vision-language models. 
In the next section, we provide empirical evidence that neural networks obey these constraints, first in a simple con model, and then in multiple large scale networks.

\section{Toy neural network implementation}\label{section:nn}
\begin{figure}[htbp]
  \centering
  \begin{minipage}{0.5\textwidth}
    \includegraphics[width=\linewidth]{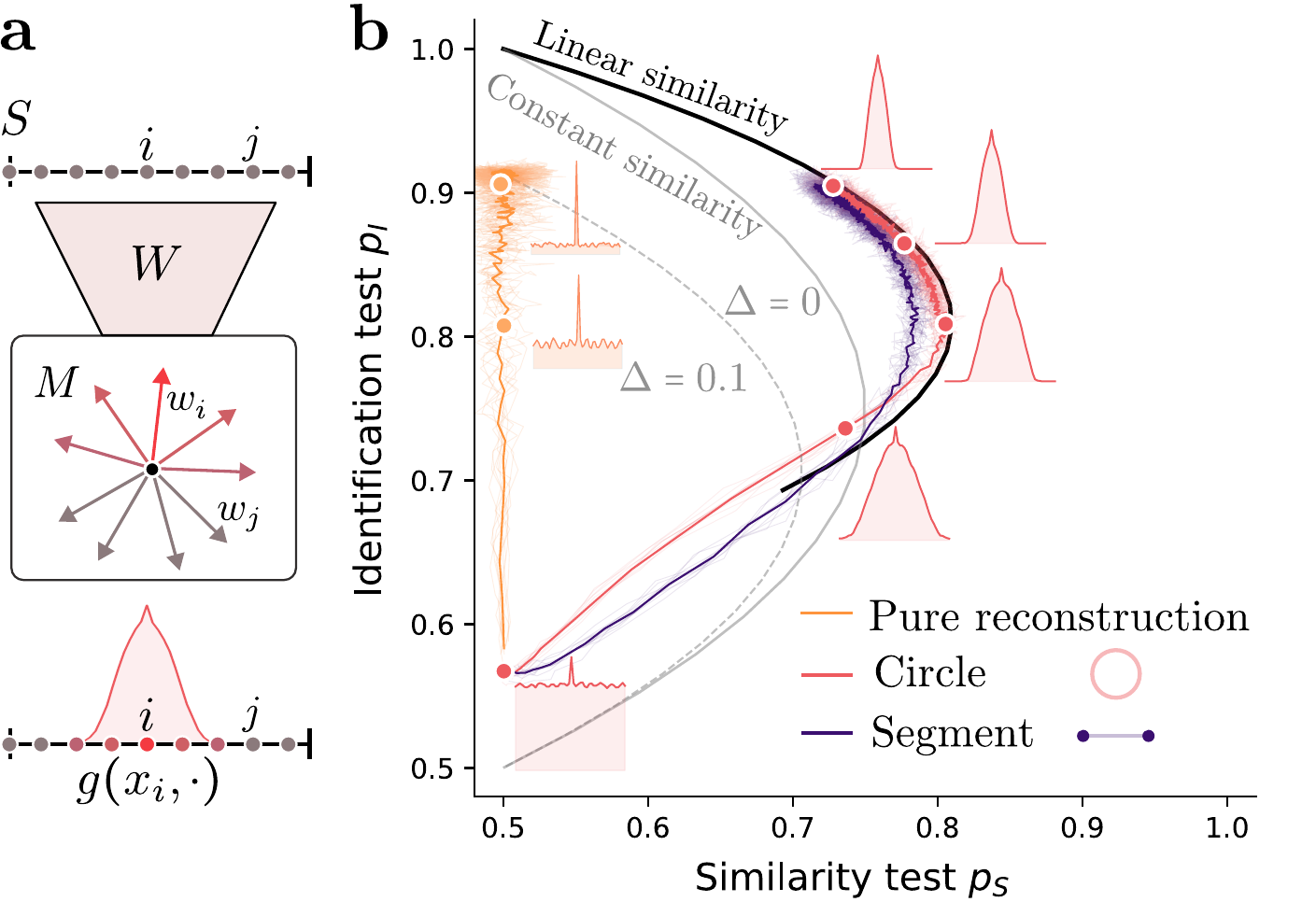}
  \end{minipage}%
  \hfill
  \begin{minipage}{0.5\textwidth}
    \captionof{figure}{\textbf{Emergent resolution and tradeoff in toy architecture.} ($p_S,p_I$) results for the toy model (\textbf{a}) of \Cref{section:nn} with 50 inputs. \textbf{b.} The orange curve shows the average training trajectory for a purely reconstruction loss. The orange insets show the learned (average) similarity function at two epochs. The gray and dashed lines show the curves of \Cref{theo:2-item} with noise levels $\Delta=0,0.1$, respectively. The red curve shows the average training trajectory when the loss is based on the similarity test on a circle while the purple one is trained on stimuli coming from a segment. The black line shows the theoretical performances obtained with linearly decaying similarity functions, as in \Cref{prop:lineardecay}.}
    \label{fig:neural_net}
  \end{minipage}
\end{figure}
We start from the toy architecture of \citet{elhage2022toy}, that permits a direct comparison with the analyses above.
The input vector $x\in \mathbb{R}_+^l$, whose entries we identify with features, is linearly encoded by $W \in \mathbb{R}^{m\times l}$, decoded by $W^\top$, and then the ReLU activation function $\sigma$ is applied elementwise $f(x)=\sigma(W^\top W x)$ (\Cref{fig:neural_net}a).
When trained with reconstruction MSE loss and sparse inputs, this model displays the phenomenon known as \emph{superposition}: features associated with input dimensions are represented as orthogonally (or \emph{dissimilarly}) as possible to minimize their interference in reconstruction \citep{elhage2022toy}.
This, in turn, means striving for good identification performance and thus the capability of processing a large number of features simultaneously.

We contrast this with the effect of inducing the model to learn representations with simple forms of metric (semantic) structure. 
To do so, we consider two spaces of stimuli made of $l$ points $\sset{x_1,\dots,x_l}$ equally spaced in the interval $[0,1]$: a (flat) circle, with distance $d(x,y)=\min(|x-y|,1-|x-y|)$, and a segment, with distance $d(x,y)=|x-y|$.
The model was trained 
to perform $3$-items similarity tests (as explained in \Cref{section:setup})
on the metric space by encoding its points, the stimuli, as $l$-dimensional one-hot vectors.
Given this last assumption, the $i$-th column of $W$, $w_i$, can be interpreted as the latent embedding of $x_i$, and the model's output $f(x_i)_j = \sigma(w_j^\top w_i) := g(x_i,x_j)$ as the non-negative similarity between $x_i$ and $x_j$. 
The model was trained to convergence 10 times and, for each epoch, and we recorded the average similarity and identification ratios $p_S,p_I$ of \Cref{eq:decision_sim,eq:decision_id} using the learned $g$.

\Cref{fig:neural_net}b shows the resulting training trajectories for three different runs in the similarity-identification plane: the \textcolor[HTML]{fea964}{orange} run corresponds to trainings with pure reconstruction loss, in \textcolor[HTML]{ed5a5f}{red} the run with pure similarity task loss on the circle and in \textcolor[HTML]{3b0f70}{purple} on the segment.
In all cases, we used $l=50$ stimuli, a hidden dimension of $m=10$ and repeated the experiment $10$ times.
Additional details of the experiment and results can be found in \Cref{appendix:extra_numerics}. 

As expected, when the network is trained only on reconstruction loss, there is no improvement in $p_S$ but a steady increase in $p_I$.
Features are arranged as orthogonally as possible but, due to the low number of hidden dimensions, some interference between them remains.
If features are arranged on a line, visualizing  the learned similarity function $g(x,\cdot)$ for a fixed $x$ at the last training step shows that it is close to being a Dirac delta on $x$, with smaller-scale random-like noise on other features.
Estimating this noise scale $\Delta$ and using that in the equations given by \Cref{theo:noise}, shows that the corresponding dashed curve accurately predicts the value of $p_I$ at which the training stops.

In contrast, when the network is trained on the semantic task, 
\Cref{fig:neural_net}b shows that (starting from the bottom left corner) both $p_S$ and $p_I$ increasing up until the \virg{boundary} is reached, after which similarity begins to decrease. %\virg{wall} 
Note that the learned similarity functions $g(x,\cdot)$ for a fixed $x=0.5$ (the red insets) exhibit a transition from noise to a semantic function that respects the structure of the circle. 
Furthermore, this structure also exhibits sensitivity to resolution: the model arranges features associated with points further than a certain threshold to have a negative inner product, which is then mapped to zero by the ReLU activation.
Moreover, we see that this resolution decreases as training progresses, resulting in an increase of $p_I$ and a decrease in $p_S$. 

Not surprisingly, the neural network does not learn constant similarity functions (defined in \Cref{section:setup}), and thus the predictions given by \Cref{theo:2-item} (in \textcolor[HTML]{828282}{gray}) only provide a qualitative prediction.
However, the \textit{learned} similarity function $g(x,\cdot)$ appears to be approximately linearly decaying with distance on the circle.
Based on this observation, we can analytically derive the values of $p_S$ and $p_I$ for linearly decaying similarities in a circle, finding formulae that approximate \Cref{theo:2-item}.
\begin{proposition}[Linear decay]\label{prop:lineardecay}
On the flat circle $[0,1]$ with $d(x,y)=\min(|x-y|,1-|x-y|)$ sampled with the uniform measure, for the two-item similarity and identification tests with linearly decaying similarity $g(x,y)= \max\left(0,1-\frac{d(x,y)}{\varepsilon}\right)$,
\begin{align}
p_S(\varepsilon)=\frac{1}{2}+b(\varepsilon)-\left(\frac{3}{2}-\log(2)\right)b(\varepsilon)^2,
\ \ \ p_I(\varepsilon)=1-(1-\log(2))b(\varepsilon),
\end{align}
with $b(\varepsilon) = 2\varepsilon,\ \varepsilon\in [0,1/2]$.
\end{proposition}
\begin{proof}
The proof can be found in \Cref{proof:lineardecay}.
\end{proof}
\Cref{fig:neural_net} shows how the resulting curve (in \textbf{black}) provides a good fit to the empirical result.
Finally, when the metric space is a segment instead of a circle (\textcolor[HTML]{3b0f70}{purple}), the heterogeneity given by the presence of the two endpoints results in an overall reduced $p_S$, as qualitatively predicted by \Cref{theo:2-item}. 

\section{Evidence of tradeoff in realistic neural networks}
\label{section:realistic-nn}
%To demonstrate that the generalization-identification tradeoff is not merely a toy-model artifact, but rather a fundamental constraint in neural networks, we conducted experiments with three realistic models of increasing complexity. 
Finally, we summarize experiments and results showing that the effects described above are also observed in networks at scale.
We report details on implementations and additional results in \Cref{appendix:extra_numerics}. 

\begin{figure}[b!]
    \centering
    \includegraphics[width=\linewidth]{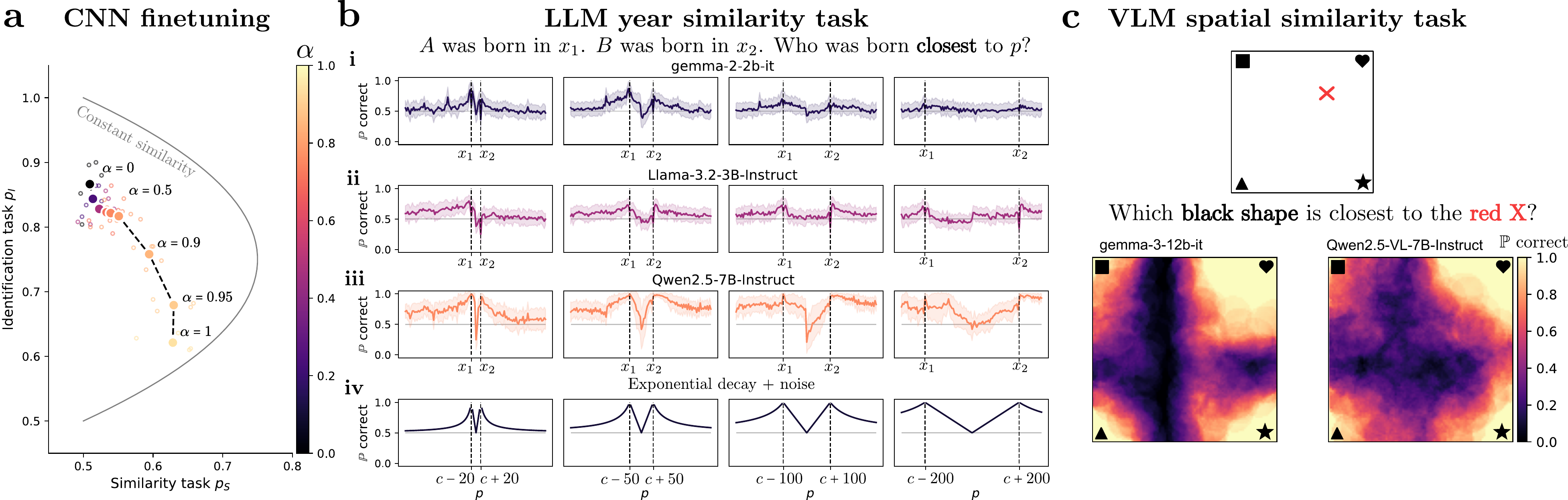}
    \caption{\textbf{Empirical resolution tradeoffs across realistic neural architectures}. 
    (a) a CNN fine-tuned on bird recognition shows tradeoff between species identification and generalization to phylogenetic similarity as a function of the weights of generalization $\alpha$ and of the resolution $\epsilon$. 
    (b) LLMs tasked with comparing years of birth show different regimes of performances, compatible with the existence of an emergent finite resolution ($\sim  70$–$80$ years). 
    (c) VLMs tasked with spatial proximity tasks show decreased accuracy beyond a model-specific resolution scale. 
    Details in \cref{appendix:extra_numerics}.}    
    \label{fig:experiments}
\end{figure}
\textbf{CNNs and evolutionary distance}
We fine-tuned a ResNet-50 model \citep{he2016deep} to analyze the generalization-identification tradeoff on bird species images \citep{wah2011caltech} using a weighted loss function $\mathcal{L} = (1-\alpha)\,\mathcal{L}_{\text{id}} + \alpha\,\mathcal{L}_{\text{sim}}$, where $\alpha$ controls the bias between identification and generalization.
Both tasks employed a triplet design ($x_1$, $x_2$, and $p$): for generalization, the model judged which reference is evolutionarily closer to the probe, using phylogenetic distances as ground truth \citep{kumar2022timetree}; for identification, it determined the reference species to which the probe belonged. 
We found that increasing $\alpha$, as a manipulation of similarity,  improved generalization while reducing identification accuracy, conforming to the relationships reported above (\Cref{fig:experiments}a). 
Models with higher $\alpha$ values consistently showed enhanced generalization, confirming the ability to manipulate this tradeoff through both training and threshold parameters.

\textbf{Year similarity task in LLMs.}
%To test whether the theoretical assumptions about resolution hold in a realistic setting, 
As a second test, we evaluated three open-source large language models (LLMs) (gemma-2b-it \citep{team2024gemma}, Llama-3.2-3B-Instruct \citep{grattafiori2024llama} and Qwen2.5-7B-Instruct \citep{yang2024qwen2}) on a similarity task requiring temporal discriminations on the scale of years. 
The models were prompted to answer questions of the following type \virg{\emph{A was born in $x_1$. B was born in $x_2$. Who was born closest to $p$?}}, where A, B are randomized names, a center year $c$ is sampled in $[1500,1700]$, $x_1 = c-\delta x,x_2=c+\delta x$ for $\delta x \in\sset{20,50,100,200}$ and $p = c + \delta p$ takes all years in $[c-300,c+300]$. 
\Cref{fig:experiments}b shows the decision curves indicating the empirical probability with which each model responded with the correct answer.  
This shows that the models' year representations closely follow our assumptions about resolution: all models showed decreased performance as probe dates moved further from reference dates, similar to what we observed with exponentially decaying similarities with noise $g(x_1,x_2)=\exp(-\mu d(x_1,x_2)) + \Delta$ (bottom row).

\textbf{Spatial similarity task in VLMs.}
Finally, we tested the effects of resolution in two Vision-Language Models (VLMs) (gemma-3-12b-it \citep{team2024gemma,gemma_techrep} and Qwen2.5-VL-7B-Instruct \citep{yang2024qwen2,qwen_techrep}), on a visual spatial  similarity task.
Four different black shapes were presented to the model in the four corners  of the image (\Cref{fig:experiments}c), together with a red cross in a random position. 
The model was tasked with indicating which black shape was closest to the red cross, and we recorded accuracy for each sampled position.
\Cref{fig:experiments}c shows that, once again, the models display clear resolution limits in their generalization capabilities, similar to those observed in the year task. 
\section{Discussion}
We have provided a formal theory of the tradeoff between identification and generalization in systems constrained by finite semantic resolution, building on the formal framework of \citet{frankland2021no}. 
Our closed form expressions reveal a universal Pareto front determined by resolution scale and stimulus geometry--a fundamental limit that is obeyed in empirical tests of model architectures both small and large.

Our analysis identifies the optimal resolution for generalization, at which semantic similarity functions tile approximately half of the representational space in discrimination tasks \citep{sorscher2022neural}. 
Beyond this point, increasing resolution impairs identification as representations become too broadly generalized.
Below it, representations are discriminable, but fail to capture meaningful similarities, thus compromising generalization. 
This offers an explanation for why both humans and state-of-the-art neural network models struggle with multi-object reasoning, despite their vast computational resources and remarkable capabilities in other domains.

The spontaneous emergence of this tradeoff across architectures--from minimal ReLU networks to vision-language models--is consistent with our analyses and our empirical findings, that are unified under the hypothesis that finite semantic resolution constitutes an information-theoretic constraint rather than implementation artifact. 
This, in turn, provides a rigorous mathematical foundation for understanding capacity limits in both artificial and biological systems. 

Our theory also indicates how competing representational strategies of intelligent systems are tied to one another: identification demands sharp, distinct representations, while generalization requires coarse, overlapping ones. 
This tension is echoed in neuroscience literature on \emph{representational efficiency} (coding related items compactly) versus \emph{processing efficiency} (handling multiple items jointly) \cite{Petri24,petri2021topological,Lesnick20}. 
Our analyses also provide a formal explanation for empirical observations in neural population coding \citep{cohen2020separability, Ganmor15}, where semantically clustered "neural thesaurus" structures emerge as optimal strategies under noise constraints, connecting to earlier models of representational redundancy \citep{curto2013combinatorial}. 

\textbf{Limitations.}
The present model assumes non-compositional representations, 
%which is appropriate for perceptual tasks, but 
which cannot capture phenomena such as hierarchical syntax, analogical reasoning, or arithmetic—where representations are formed by systematic combinations of simpler parts \cite{LakeBaroni23,FodorPylyshyn98}. 
Extending our framework to compositional coding schemes remains an important future direction.

\textbf{Future Work.}
Future work could further extend our results by: (1) using \emph{synergy–redundancy decompositions} \cite{proca2024synergistic} to examine how generalization shapes joint encoding of multiple stimuli;
(2) developing resolution-based diagnostic tools for optimizing neural architectures by targeting task-appropriate generalization-identification balance; and
(3) testing whether neural manifolds from fMRI or electrophysiology exhibit comparable resolution bounds, potentially establishing semantic resolution as a measurable link between neural geometry and behavioral generalization.

% \section*{Acknowledgements}
% M.N. acknowledges the project PNRR-NGEU, which has received funding from the MUR – DM 352/2022. G.P. acknowledges support from ... \\
%\bibliographystyle{unsrtnat}
%\bibliography{biblio}
\putbib[biblio]
%%%%%%%%%%%%%%%%%%%%%%%%%%%%%%%%%%%%%%%%%%%%%%%%%%%%%%%%%%%%
\end{bibunit}

\clearpage
\newpage
\makeatletter

\begin{bibunit}[unsrtnat]
\appendix

\ifthenelse{\boolean{includeappendix}}{

\section{Appendix / supplemental material}
\subsection{Technical details}\label{section:techinical} 

In this section, we formalize the technical aspects and assumptions required for the results of the paper.

We assume $(M,d,\Sigma,\nu)$ to be a separable metric measure space equipped with the standard metric space topology, the Borel $\sigma$-algebra $\Sigma$ generated by balls in $M$ and with a probability measure $\nu$.
This measure $\nu$, which is such that $\nu(M)=1$, determines how we are sampling stimuli from the stimulus space $M$.

In the following derivations, we will make use of two objects: 
\begin{itemize}
    \item $b_p(\varepsilon) = \nu(B_\varepsilon(p))$, the measure of the ball of radius $\varepsilon$ centered in $p$;
    \item $S_p$, the push-forward measure on $[0,\infty]$ of the distance function in $p$, $d_p(\cdot)=d(p,\cdot)$ i.e., for any measurable subset of $\mathbb{R}$, $S_p(E) = \nu(d_p^{-1}(E))$.
\end{itemize}
Note that $b_p$ is the cumulative distribution function of $S_p$ as $S_p((-\infty,\varepsilon]) = S_p([0,\varepsilon]) = b_p(\varepsilon)$ and therefore it is non-decreasing. 
We also have that $b_p(\infty):=\lim_{\epsilon\to\infty}b_p(\varepsilon)=1$.

We assume that $b_p$ is an absolutely continuous function on 
every closed sub-interval of $[0,\infty)$ i.e. such that for every $\epsilon>0$ there exists $\delta >0$ such that for any finite set of disjoint intervals $(\alpha_1,\beta_1),\dots,(\alpha_N,\beta_N)$ 
\[
\sum_{i=1}^N (\beta_i - \alpha_i) < \delta \implies \sum_{i=1}^N (b_p(\beta) - b_p(\alpha))<\epsilon.
\]
By \citet{nielsen1997introduction} (Theorem 20.10), the absolute continuity of $b_p(\varepsilon)$ implies that $S_p$ is an absolutely continuous measure w.r.t. the Lebesgue measure $\mu$.
This implies, by Radon-Nikodym theorem, that $S_p$ admits a density $f$, $S_p = \int f d\mu$, with $f(\varepsilon) = b_p'(\varepsilon)$ almost everywhere.
In fact, we can think of absolute continuity as a stronger notion of continuity, as the fundamental theorem of calculus for Lebesgue integrals (\citet{folland1999real}, Theorem 3.35) tells us that, on every interval $[c,d]$, $b_p$ is almost everywhere differentiable, and $b_p(\varepsilon)-b_p(c) = \int_c^\varepsilon b_p'(r)d\mu(r)$.

We now see how these assumptions allow us to notably simplify the derivations of the probability of success in both similarity and identification tests, while still not being too restrictive. 
In fact, most non-pathological cases of interest, like probability distributions with differentiable densities on manifolds, satisfy the assumption.

\begin{lemma}\label{lemma:only_one}
Let $X$ be the random variable of the correct answer to the $n$-item similarity or identification test $X=\mathrm{argmin}\sset{d(x_1,p),\dots,d(x_n,p)}$. If, for every $p\in M$, $b_p$ is absolutely continuous on every closed interval $[c,d]\subseteq[0,\infty)$, then $\Pro(|X|>1) = 0$.
\end{lemma}
\begin{proof}
Let $|X|$ be the cardinality of the set $X$.
\begin{align}
\Pro(|X|>1) &= \sum_{k=2}^n \binom{n}{k}\Pro(d(x_1,p)=\dots= d(x_k,p),d(x_{k+1},p)>d(x_1,p),\dots,d(x_n,p)>d(x_1,p)) \\
&\leq\sum_{k=2}^n\binom{n}{k} \Pro(d(x_1,p)=\dots =d(x_k,p)) \leq \sum_{k=2}^n\binom{n}{k} \Pro(d(x_1,p)= d(x_2,p))  \\
&=\sum_{k=2}^n \binom{n}{k} \int_{0}^\infty \Pro(d(x_1,p)=d(x_2,p)=r) d\mu(r) \leq \sum_{k=2}^n \binom{n}{k} \int_{0}^\infty \Pro(d(x_1,p)=r)d\mu(r)\\
&= \sum_{k=2}^n \binom{n}{k} \int_{0}^\infty S_p(\sset{r})d\mu(r) = 0,
\end{align}
where the last equality comes from the absolute continuity of $S_p$ w.r.t. $\mu$, as $\mu(\sset{r})=0$.
\end{proof}
This result tells us that, under the assumptions, there is a probability of $0$ that there are multiple correct answers to the similarity and identification tests.
Therefore, in the following derivations we will always only have to deal with the case $|X|=1$.

%\mnote{Don't know if I should put this one}
% We will also mention the case when $M$ is a \emph{homogeneous} space.
% Homogeneity intuitively means that, geometrically, all the points are the same. 
% In detail, this means that, for any $x,y\in M$, there exists an isometry $\phi\in \mathrm{Iso}(M)$ such that $\phi(x)=\phi(y)$.
% In that case, we always assume to be using the (unique) Haar probability measure $\nu$ associated to the group $\mathrm{Iso}(M)$, which can be proven to give to the balls a volume which depends only on the radius $\nu(B_\varepsilon(x)) = b(\varepsilon)$.  
}{
\refstepcounter{section}
\refstepcounter{subsection}
  \phantomsection
\label{section:techinical}
}

\ifthenelse{\boolean{includeappendix}}{
\subsection{Proof of \Cref{theo:2-item}}\label{appendix:proof_1}
\subsubsection{Similarity test}\label{proof:sim2}

\begin{proof}
Let us derive the probability of succeeding in the similarity test in the case of 2 items. 
The following proof will be a subcase of the more general one for $n$ items but, given its complexity, it is useful to analyze this subcase separately.

Let $x_1,x_2,p$ be sampled independently from $M$ according to the probability measure $\nu$.
Let $X(x_1,x_2,p) = \underset{i\in\sset{1,2}}{\mathrm{argmin}}\ d(x_i,p)$.
Notice that $X$ can have three different values
\[
\begin{cases}
    X(x_1,x_2,p) = \sset{1} & \text{ if } d(x_1,p)< d(x_2,p) \\
    X(x_1,x_2,p) = \sset{2} & \text{ if } d(x_2,p)< d(x_1,p) \\
    X(x_1,x_2,p) = \sset{1,2} & \text{ if } d(x_1,p) = d(x_2,p) .
\end{cases}
\]
In this last case, when $x_1,x_2$ are equidistant from $p$, any answer to the test will be correct.
By \Cref{lemma:only_one}, we only need to focus on the first two as the probability that more than one answer is correct is $0$.

We have that
\begin{equation}\label{eq:prob1}
    \Pro(Y = X) = \sum_{i=1}^2  \Pro(Y= i|X =i)\Pro(X = i).
\end{equation}
Let us now rewrite the probability $\Pro(Y =i|X=i)$ by conditioning over all possible results of the samplings of $x_1,x_2$ and the probe $p$.

For this, given the independence assumption, we assume that the event $(x_1,x_2,p)$ is an element of the measure space $(M^3,\nu^{\otimes3})$ equipped with the standard product measure.
\[
\Pro(Y= i|X=i) = \int_{M^3} \Pro(Y=i|X=i,(x_1,x_2,p))dP(x_1,x_2,p|X=i),
\]
where $dP(x_1,x_2,p|X=i)$ is the conditional measure of the sampling of $x_1,x_2,p$ given the event that $X=i$ which, by Bayes theorem, can be rewritten as
\[
dP(x_1,x_2,p|X=i) = \frac{\ones[X(x_1,x_2,p)=i]d\nu(x_1) d\nu(x_2) d\nu(p)}{\Pro(X=i)},
\]
where $\ones[X(x_1,x_2,p)=i]$ coincides with the conditional law of the (deterministic) variable $X|(x_1,x_2,p)$.

Replacing this in \Cref{eq:prob1} we get
\begin{equation}\label{eq:proof_sim_prob1}
\Pro(X=Y) = \sum_{i=1}^2 \int_{M^3} \Pro(Y=i|X = i, (x_1,x_2,p)) \ones[X(x_1,x_2,p)=i]d\nu(x_1) d\nu(x_2) d\nu(p)
\end{equation}

The independence of the samplings of $x_1$ and $x_2$ means that all indices are equally likely to be the correct answer $\Pro(Y=i|X=i)=\Pro(Y=j|X=j)\ \forall i,j\in\sset{1,2}$.

\begin{align}   
\Pro(X=Y) &= 2 \int_{M^3} \Pro(Y=1,X=1,(x_1,x_2,p))\ones[X(x_1,x_2,p)=1]d\nu(x_1)d\nu(x_2)d\nu(p)\\
&= 2 \int_M\int_M \int_{x_2\in M:d(x_2,p)> d(x_1,p)} \ \frac{g(x_1,p)}{g(x_1,p) + g(x_2,p)} d\nu(x_2) d\nu(x_1) d\nu(p) \label{eq:6} 
\end{align}

Given the fact that we are considering constant similarity functions $g(x,y) = g_{\varepsilon;0}(x,y)$ which depend only on the distance between $x$ and $y$, $g(x,y)=g(d(x,y))$, we perform the following change of coordinates $d(x_1,p)\mapsto r_1$, $d(x_2,p)\mapsto r_2$,
\begin{equation}\label{eq:proof_prob_distances}   
\Pro(X=Y) = 2\int_M d\nu(p)\int_{[0,\infty]}dS_p(r_1)\int_{(r_1,\infty]} dS_p(r_2) \frac{g(r_1)}{g(r_1) + g(r_2)},
\end{equation}
where $S_p$ is the pushforward measure induced by the distance function from the probe $p$.

We decompose \Cref{eq:proof_prob_distances} into two cases: \textbf{a.} when $r_1 > \varepsilon$ and thus both items fall outside the resolution region of the probe $p$, and \textbf{b.} when $r_1 \leq \varepsilon$ and thus the closest item falls inside.

\textbf{a.}
In the first case, given that $r_2>r_1$, we will have that both $x_1$ and $x_2$ are too far from the probe to be recognized as similar, resulting in both numerator and denominator in \Cref{eq:proof_prob_distances} to be 0. 
Here we adopt the convention $0/(0+0) = 1/2$ to describe the model being maximally uncertain in its decision.
\[
2\int_M d\nu(p) \int_{(\varepsilon,\infty]} dS_p(r_1) \int_{(r_1,\infty]}\frac{1}{2}dS_p(r_2)  = \int_M d\nu(p) \int_{(\varepsilon,\infty]}\int_{(r_1,\infty]} dS_p(r_1)dS_p(r_2)
\]

To compute this integral, we leverage the almost-everywhere differentiability of $b_p$ and apply the fundamental theorem of calculus
\begin{align} 
&\int_{(\varepsilon,\infty]}\int_{(r_1,\infty]} dS_p(r_1)dS_p(r_2) = \int_{(\varepsilon,+\infty]}(1-b_p(r_1)) dS_p(r_1) \\
&=\int_{(\varepsilon,+\infty]}(1-b_p(r_1)) b'_p(r_1) d\mu(r_1) =  \left[-\frac{(1-b_p(r_1))^2}{2}\right]_\varepsilon^\infty = \frac{(1-b_p(\varepsilon))^2}{2}. \label{eq:case_a}
\end{align}

\textbf{b.} 
When $r_1\leq\varepsilon$ the first item will be considered to be similar to the probe $g(r_1) =1$, while the second can be both similar and dissimilar.
\begin{align}
&2\int_M d\nu(p) \int_{[0,\varepsilon]} dS_p(r_1)\int_{(r_1,\infty]} dS_p(r_2) \frac{1}{1+g(r_2)} \\
&= \underbrace{2\int_M d\nu(p) \int_{[0,\varepsilon]} dS_p(r_1)\int_{(r_1,\varepsilon]} \frac{1}{2} dS_p(r_2)}_{\textbf{I}}  + \underbrace{2\int_M d\nu(p) \int_{[0,\varepsilon]} dS_p(r_1)\int_{(\varepsilon,\infty]} dS_p(r_2)}_{\textbf{II}}.
\end{align}
The term \textbf{I}, just like above, can be computed in the following way
\begin{align}
&\int_M d\nu(p)\int_{[0,\varepsilon]} dS_p(r_1)\int_{(r_1,\varepsilon]}dS_p(r_2) \\
&=  \int d\nu(p)\int_{[0,\varepsilon]} (b_p(\varepsilon)-b_p(r_1))b'_p(r_1)d\mu(r_1) \\
&= \int_M \frac{b_p(\varepsilon)^2}{2} d\nu(p) \label{eq:all_in}
\end{align}
The term \textbf{II} is simply given by $2\int_M b_p(\varepsilon)(1-b_p(\varepsilon))d\nu(p)$.

Summing together \textbf{a.} and \textbf{b.} we arrive at the following:
\begin{align}
\Pro(Y= X) &= \int_M  \frac{1}{2}(1-b_p(\varepsilon))^2 + \frac{1}{2}b_p(\varepsilon)^2+2b_p(\varepsilon)(1-b_p(\varepsilon))d\nu(p)\\
&= \int_M \frac{1}{2} + b_p(\varepsilon) - b_p(\varepsilon)^2 d\nu(p).
\end{align}
We obtain the formula for the probability of succeeding the similarity test:
\begin{equation}\label{eq:sim}
    \Pro(Y= X)=\frac{1}{2} + \langle b(\varepsilon)\rangle - \langle b(\varepsilon)^2\rangle.
\end{equation}
\end{proof}

\subsubsection{Identification test}\label{proof:ide2}
\begin{proof}
The identification test can be seen as a subset of the similarity test, in which the probe is uniformly picked among the input stimuli.
This means that the correct response will be $X(x_1,x_2,p) = \sset{i\in\sset{1,2}:x_i = p}$ and both answers will be correct only in the case that $x_1 = x_2$.

Retracing the first steps outlined in \Cref{proof:sim2}, we find that the probability of the model being correct will be
\begin{align}\label{eq:proof_ide_1}
\Pro(Y=X) &=  2\int_{M} \sum_{p\in \sset{x_1,x_2}}\frac{
1}{2} \Pro(Y = 1|X=1,(x_1,x_2,p))\ones[X(x_1,x_2,p)=1]d\nu(x_1)d\nu(x_2)\\
&=  \int_{M^2} \frac{g(x_1,x_1)}{g(x_1,x_1)+g(x_2,x_1)} d\nu(x_1)d\nu(x_1).
\end{align}
Note that we used the fact that $p=x_1$ with probability $1/2$ and $p=x_2$ with probability $1/2$.
Given that $g(x,x)=1$ and, by the definition of metric space, $d(x,y)=0 \iff x=y$, we change coordinates $d(x_1,x_2) \mapsto r$ and rewrite \Cref{eq:proof_ide_1} as
\begin{align}\label{eq:start_ide}
\Pro(X=Y) &= \int_M d\nu(x_1)\int_{(0,\infty]} \frac{1}{1+g(r)}dS_{x_1}(r).
\end{align}
When $r>\varepsilon$, the second item does not interfere with the probe $p = x_1$ and the model will choose $x_1$ with certainty, while, if $r\leq \varepsilon$, $g(r) = 1$ and it will instead be maximally uncertain.
\begin{align}
&\int_M d\nu(x_1)\int_{(0,\infty]} \frac{1}{1+g(r)}dS_{x_1}(r) = \int_M d\nu(x_1) \int_{(0,\varepsilon]}\frac{1}{2}dS_{x_1}(r) + \int d\nu(x_1) \int_{(\varepsilon,\infty] }1dS_{x_1}(r).\\
&= \frac{1}{2}\int_M b_{x_1}(\varepsilon) d\nu(x_1) + \int_M (1-b_{x_1}(\varepsilon)) d\nu(x_1) = 1 - \frac{1}{2} \langle b(\varepsilon)\rangle.
\end{align}
\end{proof}
}{
\refstepcounter{subsection}
  \phantomsection
\label{appendix:proof_1}
}

\ifthenelse{\boolean{includeappendix}}{
\subsection{Proof of \Cref{theo:noise}}\label{proof:noise}
The proof proceeds by retracting the proof of the noiseless case with some adjusted constants.

We start from the similarity test success probability, as rewritten in \Cref{eq:proof_prob_distances}.
Once again, the integral can be decomposed into two cases: \textbf{a.} when $r_1>\varepsilon$ and thus both items fall outside the resolution region of the probe $p$ and \textbf{b.} when $r_1\leq \varepsilon$ and thus the closest item falls inside.

\textbf{a.} When $r_2>r_1>\varepsilon$, we have that $g(r_1)=g(r_2)=\Delta$ and thus the ratio $g(r_1)/(g(r_1)+g(r_2)) = 1/2$ resulting in the same term of \Cref{eq:case_a} $(1-b_p(\varepsilon))^2/2$.

\textbf{b.} When $r_1 \leq \varepsilon$ and $r_2 \leq \varepsilon$ both items are similar to the probe and thus we get the same contribution of the term $b_p(\varepsilon)^2/2$ in \Cref{eq:all_in}.

The only difference from the proof of the noiseless case is when $r_1\leq\varepsilon$ and $r_2 > \varepsilon$. 
In this case, the first item is similar to the probe while the second is not, but the noise erodes the probability of the model picking the first item. 
Therefore we get the following contribution to $\Pro(Y=X)$.
\[
2\int_M d\nu(p)\int_{[0,\varepsilon]} dS_p(r_1)\int_{(\varepsilon,\infty]} dS_p(r_2) \frac{1}{1+\Delta} = \frac{2}{1+\Delta}\int_M b_p(\varepsilon)(1-b_p(\varepsilon))d\nu(p).
\]
Putting all the terms together we get
\begin{align}
\Pro(Y=X) &=\int_M \frac{1}{2}(1-b_p(\varepsilon))^2 + \frac{1}{2}b_p(\varepsilon)^2 +\frac{2}{1+\Delta} b_p(\varepsilon)(1-b_p(\varepsilon))d\nu(p) \\
&= \int_M \frac{1}{2} + \left(\frac{2}{1+\Delta}-1\right)(b_p(\varepsilon)-b_p(\varepsilon)^2) d\nu(p) \\
&= \frac{1}{2} + \frac{1-\Delta}{1+\Delta}(\langle b_p(\varepsilon)\rangle - \langle b_p(\varepsilon)^2\rangle) 
\end{align}

For the identification test, we start from \Cref{eq:start_ide}. 
Now, when $r>\varepsilon$, the second item is outside of the resolution region of $x_1$ but the noise will still make the model's decision not certain.
\begin{align}
&\Pro(Y=X) =\int_M d\nu(x_1)\int_{(0,\infty]} \frac{1}{1+g(r)}dS_{x_1}(r) \\
&= \int_M d\nu(x_1) \int_{(0,\varepsilon]}\frac{1}{2}dS_{x_1}(r) + \int d\nu(x_1) \int_{(\varepsilon,\infty] }\frac{1}{1+\Delta}dS_{x_1}(r).\\
&= \frac{1}{2}\int_M b_{x_1}(\varepsilon) d\nu(x_1) + \frac{1}{1+\Delta}\int_M (1-b_{x_1}(\varepsilon)) d\nu(x_1) = \frac{1}{1+\Delta} - \frac{1-\Delta}{2(1+\Delta)} \langle b(\varepsilon)\rangle.  
\end{align}
}{
\refstepcounter{subsection}
  \phantomsection
\label{proof:noise}
}

\ifthenelse{\boolean{includeappendix}}{
\subsection{Proof of \Cref{theo:n-item}}\label{appendix:proof_3}

\begin{lemma}\label{lemma:comb_identitiy}
\[
\sum_{j=1}^n \binom{n}{j}\frac{1}{j} x^j (1-x)^{n-j} = \sum_{j=1}^n \frac{(1-x)^{n-j}-(1-x)^n}{j}.
\]
\end{lemma}
\begin{proof}
Let us call $f_n$ the left-hand side of the identity and $g_n$ the right-hand side. 
We prove the result by showing that the generating functions of the series $(f_n)_n,(g_n)_n$ are equal.

Let us start with $(f_n)_n$.
\begin{align}
F(z)&=\sum_{n=0}^\infty f_n z^n = \sum_{n=0}^\infty\sum_{j=1}^n \binom{n}{j}\frac{1}{j} x^j (1-x)^{n-j} z^n = \sum_{j=1}^\infty \sum_{n=j}^\infty \binom{n}{j}\frac{1}{j} x^j (1-x)^{n-j}z^n\\
&= \sum_{j=1}^\infty \frac{x^j}{j} \left(\sum_{n=j}^\infty \binom{n}{j} (1-x)^{n-j} z^n\right) = \sum_{j=1}^\infty \frac{x^j}{j} \left(\sum_{k=0}^\infty \binom{j+k}{j} (1-x)^{k} z^{k+j}\right)\\ &=\sum_{j=1}^\infty \frac{x^j}{j} z^j(1-z+xz)^{-j-1} = \frac{1}{(1-z+xz)}\sum_{j=1}^\infty \frac{1}{j}\left(\frac{xz}{1-z+xz}\right)^j \\
&= - \frac{\log(1-\frac{xz}{1-z+xz})}{1-z+xz} = -\frac{\log(\frac{1-z}{1-z+xz})}{1-z+xz},
\end{align}
where we used the generating function identity for the binomial, see \citet{graham89_concrete} (Equation 5.56) and the power series expansion of $\log(1-x)$.

Let us proceed in the same way for $g_n$:
\begin{align}
G(z) &= \sum_{n=0}^\infty g_n z^n = \sum_{n=0}^\infty \sum_{j=1}^n \frac{1}{j}((1-x)^{n-j} -(1-x)^n)z^n \\
&= \sum_{j=1}^\infty \frac{1}{j}\left(\sum_{n=j}^\infty(1-x)^{n-j} z^n - \sum_{n=j}^\infty (1-x)^n z^n\right) \\
&= \sum_{j=1}^\infty \frac{1}{j}\left( (1-x)^{-j} \sum_{n=j}^\infty (z-zx)^n - \sum_{n=j}^\infty(z-zx)^n\right)\\
&=\sum_{j=1}^\infty \frac{1}{j}\left( (1-x)^{-j}  \frac{(z-zx)^j}{1-z+zx} - \frac{(z-zx)^j}{1-z+zx}\right)\\
&=\frac{1}{1-z+zx} \left(\sum_{j=1}^\infty \frac{1}{j} \left(\frac{z-zx}{1-x}\right)^j - \sum_{j=1}^\infty\frac{1}{j}(z-zx)^j \right) \\
&= \frac{-\log\left(1-\frac{z-zx}{1-x}\right) + \log(1-z+zx)}{1-z+xz} = - \frac{\log\left(\frac{1-z}{1-z+xz}\right)}{1-z+xz}
\end{align}

\end{proof}
\subsubsection{Similarity test}\label{proof:n_sim}
Recall that in the $n$-item similarity test, we are sampling independently $n$ stimuli $x_1,\dots,x_n$ and a probe $p$ and we ask the model to find which among the stimuli is the closest to $p$.
Recall that \Cref{lemma:only_one} tells us that the probability of having more than a correct answer is 0.

Retracing the first steps in \Cref{appendix:proof_1}, we find that
\begin{align}
\Pro(Y=X) &= \sum_{i=1}^n \Pro(Y=i|X=i)\Pro(X=i)\\
\end{align}
By the symmetry induced by the independence of the sampling, we see that $\Pro(Y=i|X=i) = \Pro(Y=j|X=j) = \Pro(Y=1|X=1)\ \forall i,j$ and $P(X=i)=P(X=1)\ \forall i$, and thus we can restrict to the case when the closest stimulus is the first one.
\begin{align}
&\Pro(Y=1|X=1) = \int_{M^{n}\times M} \Pro(Y=1|X=1,(x_1,\dots,x_n,p)) dP(x_1,\dots,x_n,p| X=1)\\
 &=\int_{M^{n}\times M} \Pro(Y=1|X=1,(x_1,\dots,x_n,p))\frac{\ones[X(x_1,\dots,x_n,p)=1]}{\Pro(X=1)}d\nu(x_1)\cdots d\nu(x_n)d\nu(p).
\end{align}
\begin{align}
&\Pro(Y=X) \\
&= n \int_{M^n\times M} \Pro(Y= 1|X=1,(x_1,\dots,x_n,p))\ones[X(x_1,\dots,x_n,p)=1]d\nu(x_1)\cdots d\nu(x_n)d\nu(p) \\
&=n\int_{M}d\nu(p)\int_{M}d\nu(x_1)\int_{d(x_2,p)>d(x_1,p)} \cdots \int_{d(x_n,p)>d(x_1,p)}  \frac{g(x_1,p)}{\sum_{i=1}^n g(x_i,p)} d\nu(x_2)\cdots d\nu(x_n).
\end{align}
Given that $g$ is a constant similarity function $g(x,y)=g_{\varepsilon;0}(x,y)$ which depends only on the distance between $x$ and $y$, we perform the change of coordinates $d(x_i,p) \mapsto r_i$ with $S_p$ being the pushforward measure induced by the distance function from the probe $p$.
\begin{align}
\Pro(X=Y) =n\int_M d\nu(p) \int_{[0,\infty]} dS_p(r_1)\int_{(r_1,\infty]} dS_p(r_2)\cdots\int_{(r_1,\infty]}dS_p(r_n) \frac{g(r_1)}{\sum_{i=1}^n g(r_i)}.
\end{align}
We now consider two cases separately. 
\textbf{a.} If $r_1 > \varepsilon$, no item falls close enough to the probe and thus $g(x_i,p)=0\ \forall i=1,\dots,n$ and the model's response is random:
\begin{align}
&n\int_M d\nu(p) \int_{(\varepsilon,\infty]}dS_p(r_1)\int_{(r_1,\infty]} dS_p(r_2)\cdots\int_{(r_1,\infty]}dS_p(r_n) \frac{1}{n} \\\
&= \int_M d\nu(p) \int_{(\varepsilon,\infty]} (1-b_p(r_1))^{n-1} dS_p(r_1) = \int_M d\nu(p) \int_{(\varepsilon,\infty]} (1-b_p(r_1))^{n-1} b'_p(r_1) d\mu(r_1)
\end{align}
Notice now that $b_p$ absolutely continuous implies that $(1-b_p(r_1))^{n-1}$ is absolutely continuous and, by Lebesgue's theorem, it is differentiable almost everywhere and the fundamental theorem of calculus holds (see \Cref{section:techinical}). 
We thus deduce that
\begin{align}
\int_M d\nu(p) \int_{(\varepsilon,\infty]} (1-b_p(r_1))^{n-1} b'_p(r_1) d\mu &= \int_M d\nu(p) \left[-\frac{(1-b_p(r_1))^n}{n}\right]_\varepsilon
^\infty \\
&= \int_M \frac{(1-b_p(\varepsilon))^n}{n} d\nu(p). \label{eq:n_item_proof_2}
\end{align}

\textbf{b.} If $r_1\leq\varepsilon$, then the closest stimulus is similar to the probe $g(r_1)=1$ and we write
\begin{align}\label{eq:n-item_proof_1}
n\int_Md\nu(p)\int_{[0,\varepsilon]}dS_p(r_1)\int_{(r_1,\infty]}dS_p(r_2)\cdots\int_{(r_1,\infty]}dS_p(r_n) \frac{1}{1+\sum_{i=2}^n g(r_i)}.
\end{align}
Each item $i>1$ can fall either inside of $B_\varepsilon(p)$ and contribute to the denominator of the decision function, or fall outside.
Given that the denominator only depends on the \emph{number} of stimuli which fall in $B_\varepsilon(p)$ and not on their index, we can write \Cref{eq:n-item_proof_1} as
\begin{align}
&n\int_M d\nu(p) \int_{[0,\varepsilon]} dS_p(r_1)\sum_{k=0}^{n- 1}\binom{n-1}{k}(b_p(\varepsilon )-b_p(r_1))^k(1-b_p(\varepsilon))^{n-1-k}\frac{1}{k+1} \\
&= n\int_M d\nu(p)\sum_{k=0}^{n-1}\binom{n-1}{k}(1-b_p(\varepsilon))^{n-1-k}\frac{1}{k+1}\int_{[0,\varepsilon]}(b_p(\varepsilon)-b_p(r_1))^k b_p'(r_1) d\mu(r_1) \\
&= n\int_M d\nu(p)\sum_{k=0}^{n-1}\binom{n-1}{k}(1-b_p(\varepsilon))^{n-1-k}\frac{1}{k+1} \frac{b_p(\varepsilon)^{k+1}}{k+1}\\
&= n\int_M d\nu(p)\sum_{k=0}^{n-1}\binom{n-1}{k}\frac{1}{(k+1)^2}(1-b_p(\varepsilon))^{n-1-k} b_p(\varepsilon)^{k+1}\\
&= \int_M d\nu(p)\sum_{j=1}^{n}\binom{n}{j}\frac{1}{j}(1-b_p(\varepsilon))^{n-j} b_p(\varepsilon)^{j},
\end{align}
where the last is performed by re-indexing $j=k+1$ and applying the property of the binomial coefficient $\binom{n-1}{j-1} = \frac{j}{n}\binom{n}{j}$.
Applying \Cref{lemma:comb_identitiy}, we rewrite the result in a more convenient form
\begin{equation}\label{eq:n_item_proof_3}
\int_M \sum_{k=1}^n \frac{(1-b_p(\varepsilon))^{n-k} - (1-b_p(\varepsilon))^n}{k} d\nu(p).
\end{equation}
and summing \Cref{eq:n_item_proof_2} with \Cref{eq:n_item_proof_3}, we get our final expression
\[
p^n_S(\varepsilon) = \mathbb{E}_{p\sim \nu}\left[\frac{1}{n} + \sum_{k=1}^{n-1} \frac{{}(1-b_p(\varepsilon))^{n-k} - (1-b_p(\varepsilon))^n}{k} \right].
\]

\subsubsection{Identification test}
Re-tracing the first steps of \Cref{proof:n_sim} and \Cref{proof:ide2}
\begin{align}
\Pro(X=Y)&=n\int_M\sum_{p\in\sset{x_1,\dots,x_n}} \frac{1}{n} \Pro(Y=1|X=1,(x_1,\dots,x_n,p))\ones[X(x_1,\dots,x_n,p)=1]d\nu(x_1)\cdots d\nu(x_n)\\
&= \int_M d\nu(x_1) \int_{M^{n-1}} \frac{g(x_1,x_1)}{g(x_1,x_1)+\sum_{i=2}^n g(x_i,x_1)}
 d\nu(x_2)\cdots d\nu(x_n) \\
&= \int_M d\nu(x_1) \int_{(0,\infty]} dS_p(r_2)\cdots \int_{(0,\infty]} dS_p(r_n) \frac{1}{1 + \sum_{i=2}^n g(r_i)}.
\end{align}
Just like we saw in the proof of the similarity test, here any stimulus different from the probe will contribute to the denominator of the decision function if and only if it falls in $B_\varepsilon(x_1)$. 
Moreover, the decision function depends only on the number of such stimuli and not on which ones contribute to the denominator.
Therefore, we can write
\begin{align}
\Pro(X=Y) &= \int_M d\nu(x_1)\sum_{k=0}^{n-1}\binom{n-1}{k} b_{x_1}(\varepsilon)^k(1-b_{x_1}(\varepsilon))^{n-1-k}\frac{1}{k+1}\\
&= \int_M d\nu(x_1) \frac{j}{n}\sum_{j=1}^n \frac{
1
}{j}\binom{n}{j} b_{x_1}(\varepsilon)^{j-1}(1-b_{x_1}(\varepsilon))^{n-j}\\
&= \mathbb{E}_{p\sim \nu}\left[\frac{1-(1-b_{p}(\varepsilon))^n}{n b_{p}(\varepsilon)}\right],
\end{align}
where we used the property of the binomial coefficient $\binom{n-1}{j-1} = \frac{j}{n}\binom{n}{j}$.
}{
\refstepcounter{subsection}
  \phantomsection
\label{appendix:proof_3}
}

\ifthenelse{\boolean{includeappendix}}{
\subsection{Proof of \Cref{prop:lineardecay}}\label{proof:lineardecay}
We want to compute $p_S$ and $p_I$ for the uniform measure on the flat circle $M=[0,1]$ with $d(x,y)=\min(|x-y|,1-|x-y|)$ for the linearly decaying similarity function with resolution $\varepsilon$, $g(r)=\sigma\left(1-\frac{r}{\varepsilon}\right)$, where $\sigma(x) = \max(x,0)$.

First, note that for the uniform measure, we have that
\[
b_x(\varepsilon)=\nu(B_\varepsilon(x)) = 
\begin{cases}
    2\varepsilon & \text{ if } \varepsilon
    \leq \frac{1}{2}\\
    1 & \text{ if } \varepsilon>\frac{1}{2}
\end{cases}=b(\varepsilon),
\]
i.e. the length of the interval $[-\varepsilon,\varepsilon]$ on the circle.
Accordingly, we have that the measure $S_x$ is such that
\[
S_x(E) = S(E)  \int_E b'(r) d\mu(r) = 2 \mu(E),
\]
if $E\subseteq [0,\frac{1}{2}]$.
\paragraph{Similarity test}
We start from \Cref{eq:proof_prob_distances}
and, once again, consider the different cases.
If $r_1,r_2 > \varepsilon$, there is no difference from the constant case: the probe has similarity $0$ with both $x_1$ and $x_2$, therefore the model is maximally uncertain. 
This term will contribute $(1-b(\varepsilon))^2/2$ to $\Pro(Y=X)$.

If $r_1\leq \varepsilon$ and $r_2>\varepsilon$, there is no difference from the constant case as the probe is similar to $x_1$ with no interference from $x_2$.
We get a contribution of $2b(\varepsilon)(1-b(\varepsilon))$.

If $r_1\leq\varepsilon,r_2\leq\varepsilon$, we need to compute
\begin{align}
&2\int_{[0,\varepsilon]} dS(r_1)\int_{(r_1,\varepsilon]} dS(r_2) \frac{g(r_1)}{g(r_1)+g(r_2)} = 8\int_{[0,\varepsilon]}\int_{(r_1,\varepsilon]} \frac{1-r_1/\varepsilon}{1-r_1/\varepsilon+1-r_2/\varepsilon}d\mu(r_1)d\mu(r_2) \\
&= 8 \int_{[0,\varepsilon]} (\varepsilon-r_1)\log(2) d\mu(r_2) = 8\cdot\frac{1}{2}\varepsilon^2\log(2)=(2\varepsilon)^2\log(2 )=\log(2)b(\varepsilon)^2.
\end{align}
Putting together the three contributions, we get
\[
\Pro(Y=X)= \frac{1}{2} -b(\varepsilon)+\frac{1}{2}b(\varepsilon)^2+2b(\varepsilon)-2b(\varepsilon)^2+\log(2)b(\varepsilon)^2= \frac{1}{2}+b(\varepsilon) -(3/2-\log(2))b(\varepsilon)^2.
\]
\paragraph{Identification test}
We start from \Cref{eq:start_ide} and consider two cases.
If $r>\varepsilon$, then $x_2$ does not interfere with the probe and thus the model will choose $x_1$ with probability $1$. 
Just like the constant case, we get a contribution of $1-b(\varepsilon)$.

If $r\leq \varepsilon$, we need to compute
\begin{align}
\int_{(0,\varepsilon]} \frac{1}{1+g(r)} dS(r)=\int_{(0,\varepsilon]} 2\frac{1}{1+1-r/\varepsilon} d\mu(r) = 2\log(2)\varepsilon = \log(2)b(\varepsilon).
\end{align}
In total, we get 
\[
\Pro(Y=X) = 1-b(\varepsilon) + \log(2)b(\varepsilon)=1-(1-\log(2))b(\varepsilon).
\]
}{
\refstepcounter{subsection}
  \phantomsection
\label{proof:lineardecay}
}

\ifthenelse{\boolean{includeappendix}}{
\subsection{Details on numerical experiments}\label{appendix:extra_numerics}
All the code used to produce the results can be found in \url{https://github.com/nplresearch/generalization}.
\subsubsection{Toy model}
\label{appendix:extra_numerics:toy-model}
The architecture of the toy model we used is the following linear bias-less autoencoder with a nonlinearity at the end
\[
f(x)=\sigma(W^\top W x),
\]
where $\sigma$ is the ReLU activation function $\sigma(x) = \max(x,0)$ and $W\in \mathbb{R}^{m\times l}$.

In both the pure-reconstruction and semantic experiments, the inputs were chosen to be $l=50$ one-hot vectors $x = e_i\ \forall i=1,\dots,m$.
The hidden space dimension was chosen to be $m=10$.

The pure reconstruction experiment is performed by minimizing the MSE loss between input one-hot and its reconstruction through the network
\[
L_{\mathrm{rec}}=\sum_{i=1}^l \norm{e_i - \sigma(W^\top W e_i)}^2 = \sum_{i=1}^l \norm{e_i - \sigma(W^\top w_i)}^2.
\]

In the semantic case, the loss is built in the following way.
Three different indices $i,j,k$ are picked randomly and their associated one-hots are built $e_i,e_j,e_k$.
Then, we compute the ratio of similarities 
\[
D_i = \frac{\sigma(w_i^\top w_k)}{\sigma(w_i^\top w_k)+\sigma(w_j^\top w_k)},\ D_j = \frac{\sigma(w_j^\top w_k)}{\sigma(w_i^\top w_k)+\sigma(w_j^\top w_k)}.
\]
The index $\hat{i}\in\sset{i,j}$ of the correct answer is computed by taking the minimum between $d(x_i,x_k)$ and $d(x_j,x_k)$, where the distance function is given as a training input in the form of a distance matrix.
The loss, finally, is computed by taking the Negative Log Likelihood Loss (NLL) between the distribution $(D_i,D_j)$ and the one hot vector encoding the correct response.
\[
L_{\mathrm{sim}} = -\frac{1}{2} D_{\hat{i}}.
\]

For all experiments, each epoch is made of 2000 samples, with batch size 128.
The models are trained for 500 epochs with the Adam optimizer, with learning rate 0.0007 and 0 weight decay.

Given that random vectors in high-dimensional space tend to be close to orthogonal, biasing the model towards high $p_I$, we initialize the weight matrix with i.i.d. uniform
in the interval $[0,2]$.

At each epoch, the model is evaluated by performing similarity and identification tests.
1,000 triplets $(i,j,k)$ ($k\in\sset{i,j}$ for the identification) are extracted, and the average $D_{\hat{i}}$ is recorded to obtain the values of $p_S$ and $p_I$ shown in \Cref{fig:neural_net}.
The average similarity functions shown in the figure's insets are obtained as $g_i(j)=\sigma(w_i^\top w_j)$ for every $j\in 1,\dots,l$.
Leveraging the symmetry of the circular structure, each vector $g_i$ is circularly shifted so that the index $i$ goes to the center of the circle $g_i \mapsto \tilde{g}_{i}$.
Finally, we take the average over $i$, $\tilde{g} = \frac{1}{l}\sum_{i=1}^l \tilde{g}_i$.

We show the distance matrices for the circle and line experiments, together with the full learned similarity matrices for a single run in \Cref{fig:t_details_1}.

In \Cref{fig:nn_diff_dims}, moreover, we see the results of the three different trainings for three values of the neural network's latent dimension.
As it increases, we see how the model is able to have less interference between representations, signified by $p_I$ being able to reach higher values.
Visualizing the average learned similarity functions and estimating the noise value, we are able in all cases to predict the maximum $p_I$ using \Cref{theo:noise}.

\begin{figure}
    \centering
    \includegraphics[width=0.95\linewidth]{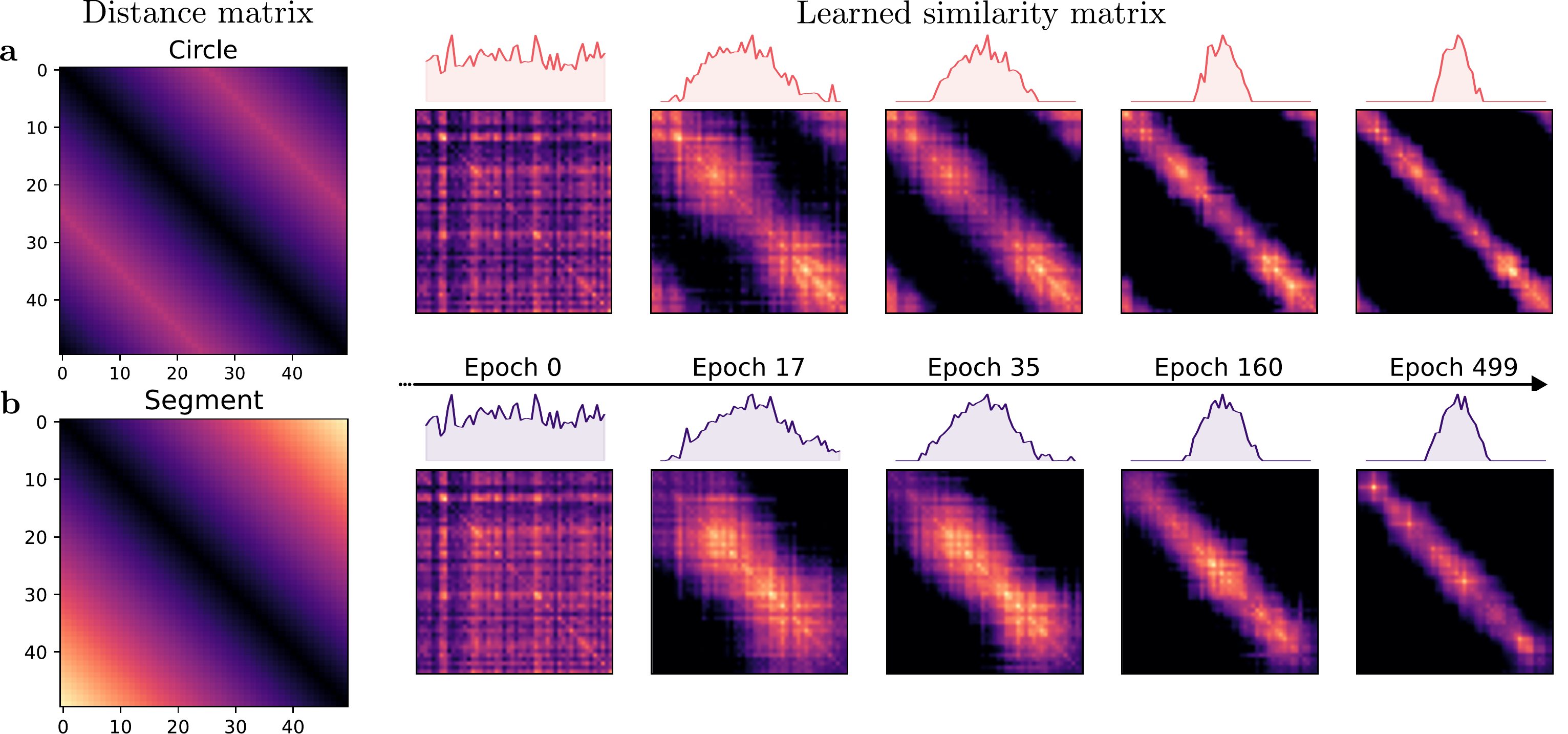}
    \caption{Visualization of the distance matrix (left) and the learned similarity matrices through training for the circle (top row) and the segment (bottom row).}
    \label{fig:t_details_1}
\end{figure}

\begin{figure}
    \centering
    \includegraphics[width=\linewidth]{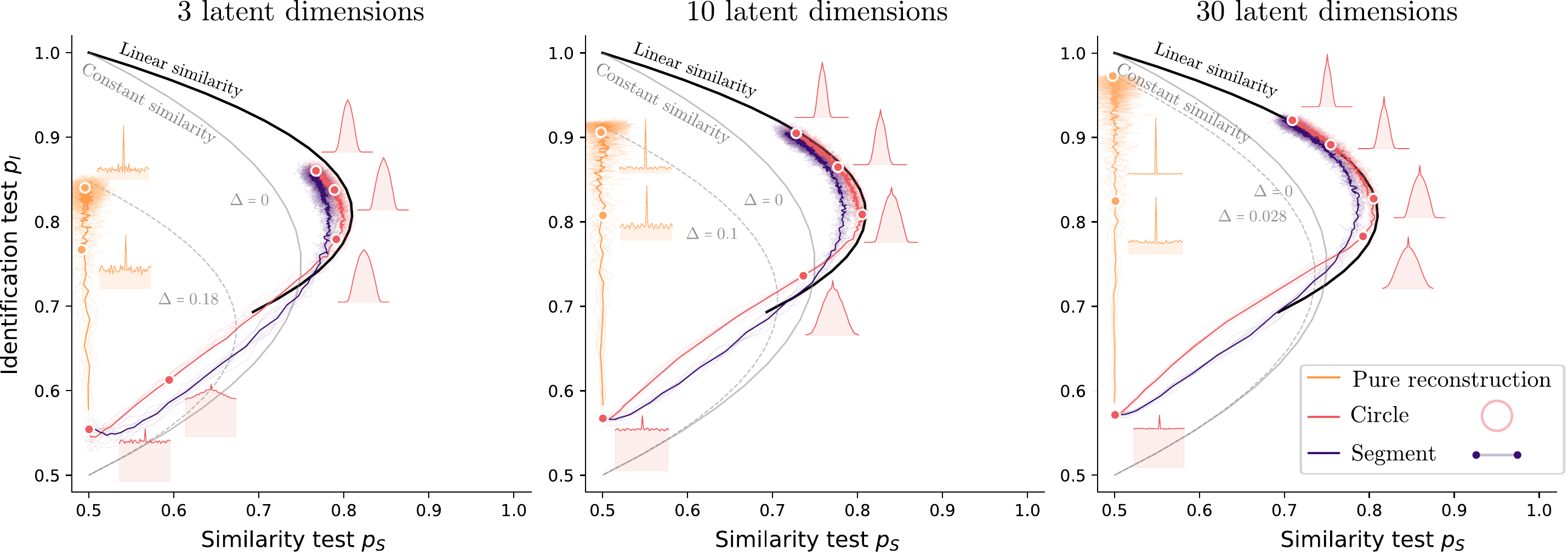}
    \caption{Different training trajectories of the toy model with different latent dimensions, visualized as in \Cref{fig:neural_net}.}
    \label{fig:nn_diff_dims}
\end{figure}
}{
\refstepcounter{subsection}
  \phantomsection
\label{appendix:extra_numerics}
}

%%%%%%%%%%%%%%%%%%%%%%%%%%%%%%%%%%%%%%%%%%%%%%%%%%%%%%%%%%%%

\subsubsection{Convolutional Neural Network fine-tuned on evolutionary distances among bird species}
\label{appendix:extra_numerics:birds}
\paragraph{Experimental setup.} To test our theoretical predictions in a realistic computer vision setting, we fine-tuned a ResNet-50 model \citep{he2016deep} pre-trained on ImageNet.
We used the Caltech-UCSD Birds-200-2011 dataset \citep{wah2011caltech}, which contains 11,788 images of 200 bird species, paired with evolutionary distance data from the TimeTree database \citep{kumar2022timetree}.
The experimental design involved two tasks with a consistent triplet-based evaluation format:
\begin{itemize}
\item \textbf{Identification task:} Given images of two reference species ($x_1, x_2$) and a probe image, determine which reference species the probe belongs to.
\item \textbf{Similarity task:} Given images of two reference species ($x_1, x_2$) and a probe species ($p$), determine which reference species is evolutionarily closer to the probe.
\end{itemize}
Using a contrastive loss that encouraged embedding bird images closer to their evolutionary relatives, we fine-tuned the model using a composite loss:
\[
\mathcal{L} = (1-\alpha)\,\mathcal{L}_{\text{id}} + \alpha\,\mathcal{L}_{\text{sim}},
\]
where $\mathcal{L}_{\text{id}}$ is a cross-entropy loss for species identification, and $\mathcal{L}_{\text{sim}}$ aligns the embedding space with evolutionary distances.
The parameter $\alpha$ controls the balance between identification and generalization objectives.
During evaluation, we defined similarity using a threshold $\epsilon$ on feature distances, where distances below $\epsilon$ indicated similarity.
This allowed us to systematically study the generalization-identification tradeoff by varying both $\alpha$ and $\epsilon$.

\begin{figure}[h]
    \centering
    \includegraphics[width=\linewidth]{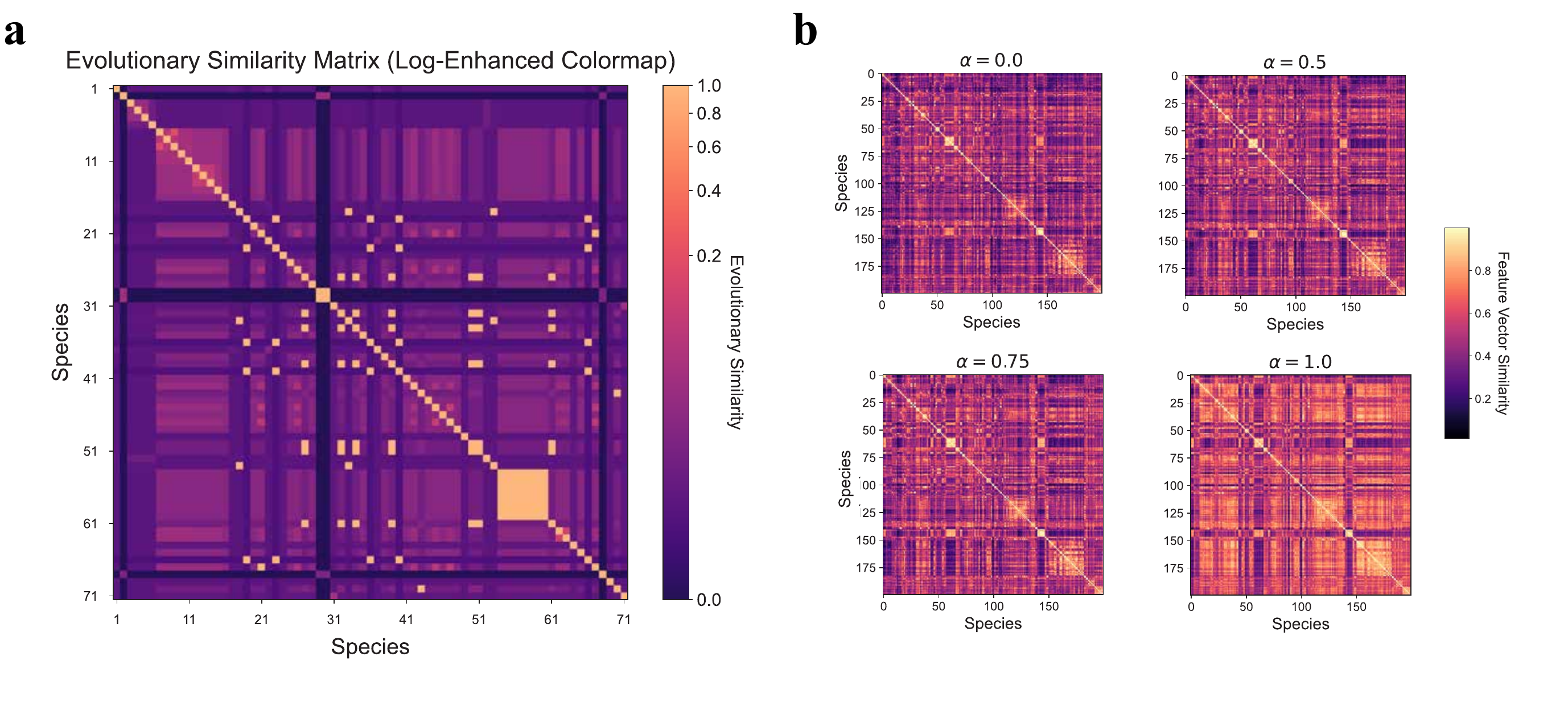}
    \caption{Evolutionary similarity between species obtained from (a) bird phylogeny and (b) the feature vector similarities as $\alpha$ is tuned.}
    \label{fig:cnn_loss}
\end{figure}

\paragraph{Training details.} We trained the model for 15 epochs using SGD with momentum 0.9, weight decay $1e-4$, and an initial learning rate of 0.001, reduced by a factor of 0.1 when validation performance plateaued. 
To handle GPU memory constraints, we used a batch size of 8 with gradient accumulation over 4 steps (effective batch size 32). 
We tested $\alpha$ values ranging from 0.0 to 1.0 with several random seeds (42-46) to ensure robust results. 

The birds dataset was split 64-16-20\% for training, validation, and testing, with an additional 15\% of species held out completely as out-of-distribution test data.
The evolutionary distance loss ($\mathcal{L}_{\text{sim}}$) was implemented by computing pairwise distances in feature space and aligning them with normalized evolutionary distances derived from the phylogenetic tree. 
This explicitly encouraged the CNN to map visual features into a space that preserved evolutionary relationships as shown in Figure \ref{fig:cnn_loss}.

\paragraph{Theoretical connections.} Our experimental framework directly maps to the theoretical constructs in Miller's Law. 
The identification task measures $p_I$ (probability of correct identification), while the similarity task measures $p_S$ (probability of correct similarity judgment). 
The threshold $\varepsilon$ corresponds to the resolution parameter in our theoretical framework, controlling the ball measure $b(\varepsilon)$ that determines which items are considered similar.

\paragraph{Evolution during training.} We monitored how the identification-generalization constraints evolved during training by tracking both scores across epochs. 
With $\alpha=0$ (pure identification objective), models rapidly optimized for identification at the expense of generalization.
As $\alpha$ increased, especially beyond 0.5, models traced distinct trajectories through $(p_s, p_I)$ (or G-I) space, with higher $\alpha$ values showing earlier and more pronounced shifts toward generalization.
%Interestingly, at intermediate stages of training (epochs 5-10), models exhibited temporary improvements in both metrics simultaneously, suggesting that initial feature learning benefits both tasks before the inherent tradeoffs become dominant.

The final equilibrium position in G-I space was primarily determined by $\alpha$, with higher $\alpha$ values reliably producing models with better generalization capabilities.
Out-of-distribution testing revealed that models with higher $\alpha$ values demonstrated substantially better generalization to unseen bird species, confirming that the similarity-based training objective promotes more robust feature learning that captures fundamental biological relationships rather than superficial correlations.

\begin{figure}[h]
    \centering
    \includegraphics[width=0.95\linewidth]{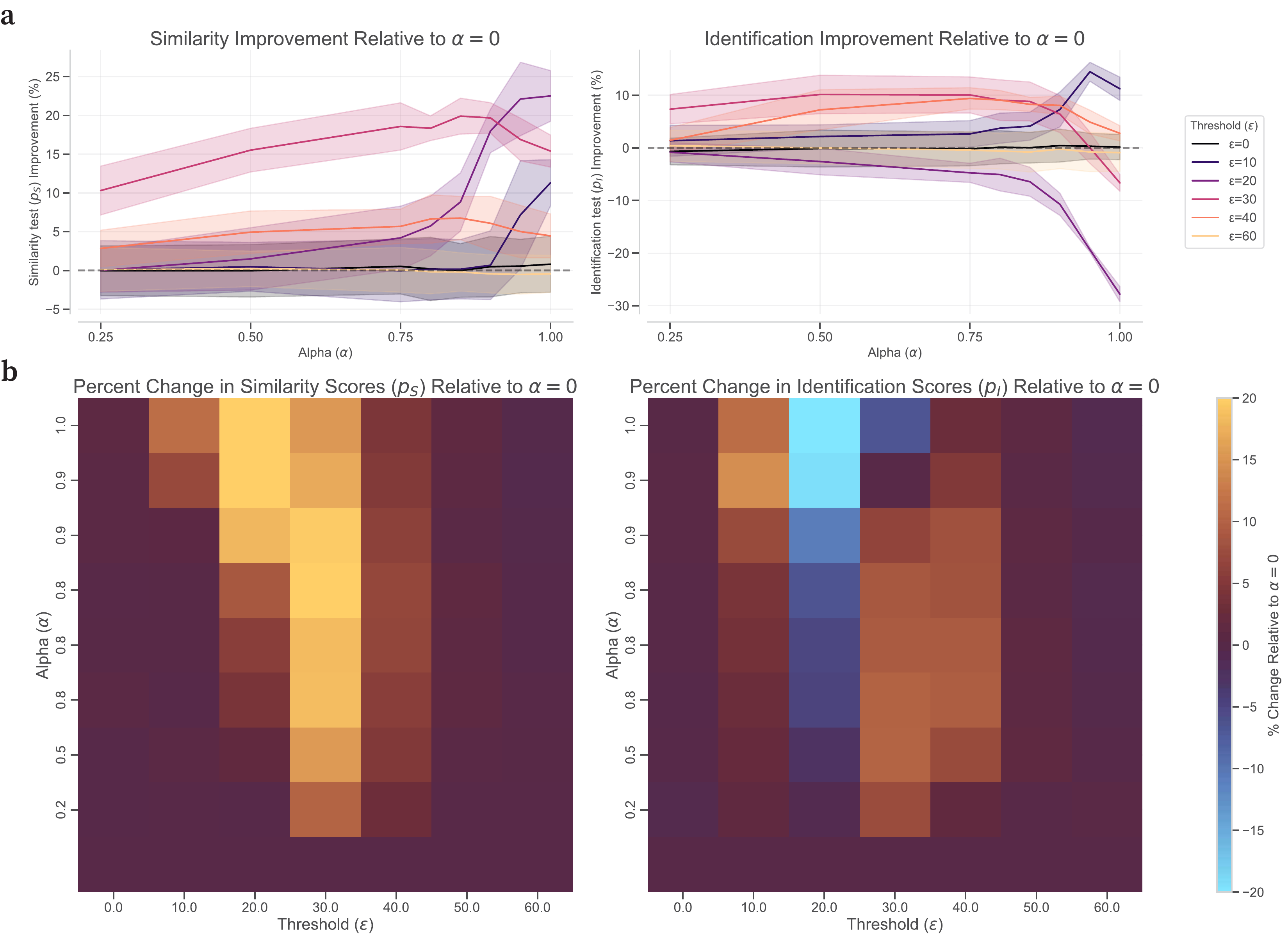}
    \caption{(a) Average generalization and identification performance for  different fixed  thresholds ($\varepsilon$), showing how threshold choice affects the generalization-identification trade-off.
    (b) results for parameter space swipe in both $(\varepsilon, \alpha)$, highlighting the presence of a narrow band of $\varepsilon$ values for which both generalization and identification show relative improvements with respect to $\alpha=0$.}
    \label{fig:cnn_improv}
\end{figure}

\paragraph{Results.} As shown in Figure~\ref{fig:experiments}a, the bird CNN exhibits a clear tradeoff between generalization and identification.
We expand these results in Figure \ref{fig:cnn_improv}, in which we show the performance improvement for generalization (a, left) and the the decrease (a, right) in identification relative to the baseline $\alpha=0$ (pure identification), as a function of $\alpha$ and a few fixed values of $\varepsilon$. 
Figure \ref{fig:cnn_improv}b provides a parameter scan, highlighting the presence of a critical scale that provides maximal improvements on both tasks. 
In fact, for low thresholds ($\varepsilon \leq 10$) and high thresholds ($\varepsilon \geq 50$), generalization performance remains similar across all $\alpha$ values, with minimal impact on identification.
For threshold $\varepsilon = 20$, we see generalization performance improvement at the cost of identification performance. 
For threshold $\varepsilon = 30$, we see score improvements for both tasks, which are then lost for larger $\varepsilon$ values. 
Indeed, for thresholds between ($20 < \varepsilon < 50$), we observe significant generalization improvements for higher $\alpha$ values, accompanied by corresponding identification performance decreases, especially in the $\alpha \geq 0.8$ range.
This confirms our theoretical prediction that increasing resolution (larger $\varepsilon$) shifts the balance toward better generalization at the expense of identification accuracy.
% }{
% \refstepcounter{subsection}
%   \phantomsection
% \label{appendix:extra_numerics:birds}
% }

\subsubsection{LLMs performing date-of-birth identification vs similarity task}
\label{appendix:extra_numerics:llm}

\paragraph{Evidence of resolution}\label{sub:res_llm}
\paragraph{Experimental setup.} We investigated whether large language models exhibit semantic resolution when processing time information. 
We tested three models: gemma-2-2b-it \citep{team2024gemma}, Llama-3.2-3B-Instruct \citep{grattafiori2024llama}, Qwen2.5-7B-Instruct \citep{bai2023qwen} on the following task.
The models are fed the system prompt \texttt{"You are a useful chatbot assistant."} and are asked to respond to the prompt \texttt{"A was born in x. B was born in y. Who was born closest to p? Answer with a single name."}

The variables \texttt{A},\texttt{B},\texttt{x},\texttt{y} and \texttt{p} are generated in the following way:
\begin{enumerate}
    \item A central year \texttt{c} is sampled uniformly from the set of integers $\sset{1500,1501,\dots,1699}$;
    \item For each value $\delta x \in \sset{20,50,100,200}$, we fix \texttt{x} $=$ \texttt{c} $-\delta x$ and  \texttt{y} $=$ \texttt{c} $+\delta x$ with probability 0.5 and \texttt{x} $=$ \texttt{c} $+\delta x$ and  \texttt{y} $=$ \texttt{c} $-\delta x$ with probability 0.5.
    \item For each pair \texttt{x},\texttt{y} chosen in this way, we run the prompt with every $\texttt{p}=\texttt{c}+\delta \texttt{p}, \delta\texttt{p}\in\{\texttt{c}-300,\texttt{c}+300\}$.
    \item The prompt with each value of \texttt{p} is ran 20 times, randomizing the variables \texttt{A} and \texttt{B}, which are two different names sampled from the list \texttt{[
    "Alice", "Bob", "Charlie", "David", "Eve", "Frank", "Grace", "Heidi",
    "Ivan", "Judy", "Karl", "Liam", "Mallory", "Nina", "Oscar", "Peggy",
    "Quentin", "Rupert", "Sybil", "Trent", "Uma", "Victor", "Walter",
    "Xander", "Yvonne", "Zach", "Abigail", "Benjamin", "Catherine", "Daniel",
    "Elena", "Frederick", "Gabriella", "Henry", "Isabella", "Jack", "Katherine",
    "Lucas", "Mia", "Nathan", "Olivia"
]}.
\end{enumerate}
If the answer of the model belongs to the set of sampled names, then we check whether it is equal to the name associated to the smallest year.

We repeat the process, sampling \texttt{c} 40 times and averaging the results.
What we obtain is a function from the probe displacement w.r.t. \texttt{c},  $\delta \texttt{p}\in [-300,300]$ to the probability of the model decision function $\mathbb{E}_c[D_1(\texttt{x},\texttt{y})]\in [0,1]$.

\paragraph{Results.} Figure~\ref{fig:experiments}b in the main text displays the probability of correct answers for both models across date displacements. 
Several observations support our theory:

\begin{enumerate}
    \item Both models show high performance when probe dates are near reference dates (small displacements), but performance degrades as displacement increases.
    \item The pattern follows our theoretical assumptions: the performance approaches chance level (0.5) when the reference years are close and the probe falls between them and when the reference years are far and probes is far form both.
    
\end{enumerate}

\paragraph{Similarity and identification tasks}
\begin{figure}
    \centering
    \includegraphics[width=0.7\linewidth]{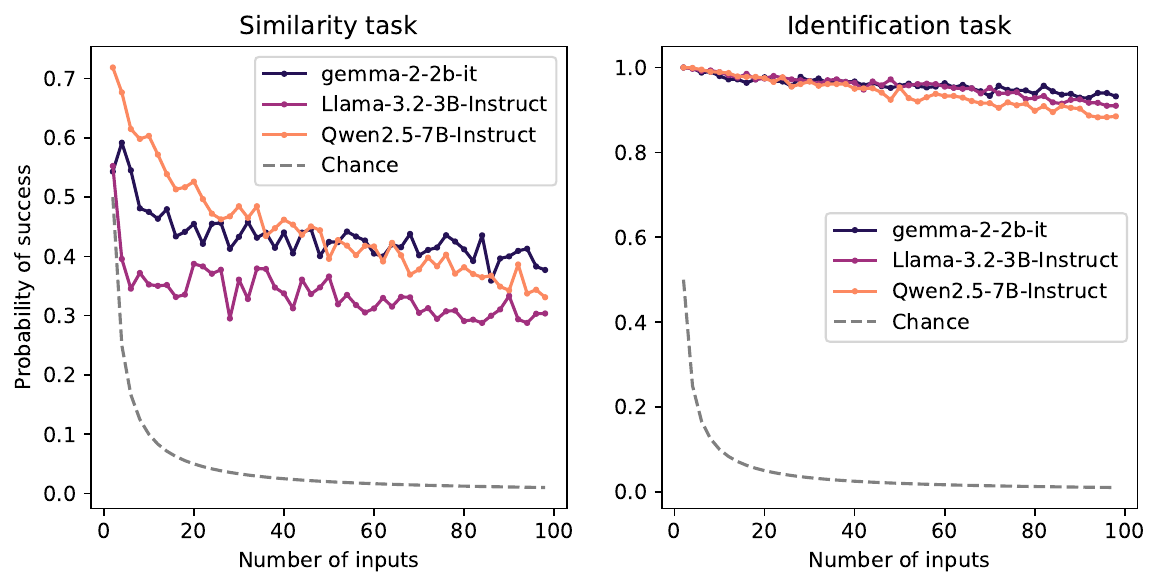}
    \caption{Similarity and identification performances of three LLMs on the interval of years $[800,1599]$}
    \label{fig:exp_years_performance}
\end{figure}
\paragraph{Experimental setup.} We then performed similarity and identification tasks to gauge the performance of these three models as the number of inputs provided increases.

For each number of inputs $n\in\sset{2,4,\dots,100}$, we sample 1,000 prompts built in the following way
\begin{itemize}
    \item \textbf{Similarity task:} \texttt{"A1 was born in x1. A2 was born in x2."}$+\dots+$\texttt{" An was born in xn." Who was born closest to p? Answer with a single name.}
    \item \textbf{Identification task:} \texttt{"A1 was born in x1. A2 was born in x2."}$+\dots+$\texttt{" An was born in xn." Who was born in p? Answer with a single name.}
\end{itemize}
Here \texttt{A1}$,\dots,$\texttt{An} are names sampled from a list of 200 names similar to the one described above, and $\texttt{x1},\dots,\texttt{xn}$ are random integers in the interval $[800,1599]$.
The probe \texttt{p} is a random integer in the same interval for the similarity task and, for the identification tasks, it is randomly chosen from the set $\sset{\texttt{x1},\dots,\texttt{xn}}$.

\paragraph{Results.}
In \Cref{fig:exp_years_performance}, we plot the performances obtained for the three models.
Overall, we see that all models perform well on the identification task, with its performance decreasing with a small rate.
Instead, for the similarity task we see how the performances are definitely worse, never being greater than 0.7 but decreasing in a much graceful way than the $1/n$ of random chance.

Interestingly, if we focus on Gemma and Qwen, we are able to qualitatively observe the same behavior of the theoretical model in \Cref{fig:n_items} of the main text.
In fact, it appears that Qwen is favouring generalization, resulting in a good similarity task performance for a low number of items but a steeper decrease for increasing $n$. Gemma, instead, achieves close to chance similarity task performances when $n$ is small but decreases less rapidly when $n$ increases.
If we map this to our theoretical investigation, it appears that Gemma is adopting a smaller $\varepsilon$ than Qwen, a conjecture which is corroborated by the identification test performance decreasing faster for the latter.
% }{
% \refstepcounter{subsection}
%   \phantomsection
% \label{appendix:extra_numerics:llm}
% }

% \ifthenelse{\boolean{includeappendix}}{
\subsubsection{VLM tasks}
\label{appendix:extra_numerics:vlm}

To assess the presence of finite semantic resolution in vision–language models (VLMs), we designed several spatial similarity/identification tasks using synthetic images. Two VLMs were evaluated: \texttt{gemma-3-12b-it} and \texttt{Qwen2.5-VL-7B-Instruct}. Besides collecting the models' textual responses, we also logged the scores (logits) of selected, task-depending tokens, to inspect how each model ranks different token choices before softmax.

\paragraph{Evidence of resolution}
\paragraph{Dataset.}
We generated 1,000 images, each featuring four black stencils positioned at fixed locations on a white background. The stencils, chosen randomly between square, triangle, heart, star, varied specific position across images. Additionally, each image included a randomly placed red X, designated as the "target" (see Figure~\ref{fig:wurstelony}). We logged the distance between the target and each stencil.

\begin{figure}
    \centering
    \begin{minipage}{0.2\textwidth}
        \centering
        \fbox{\includegraphics[width=\linewidth]{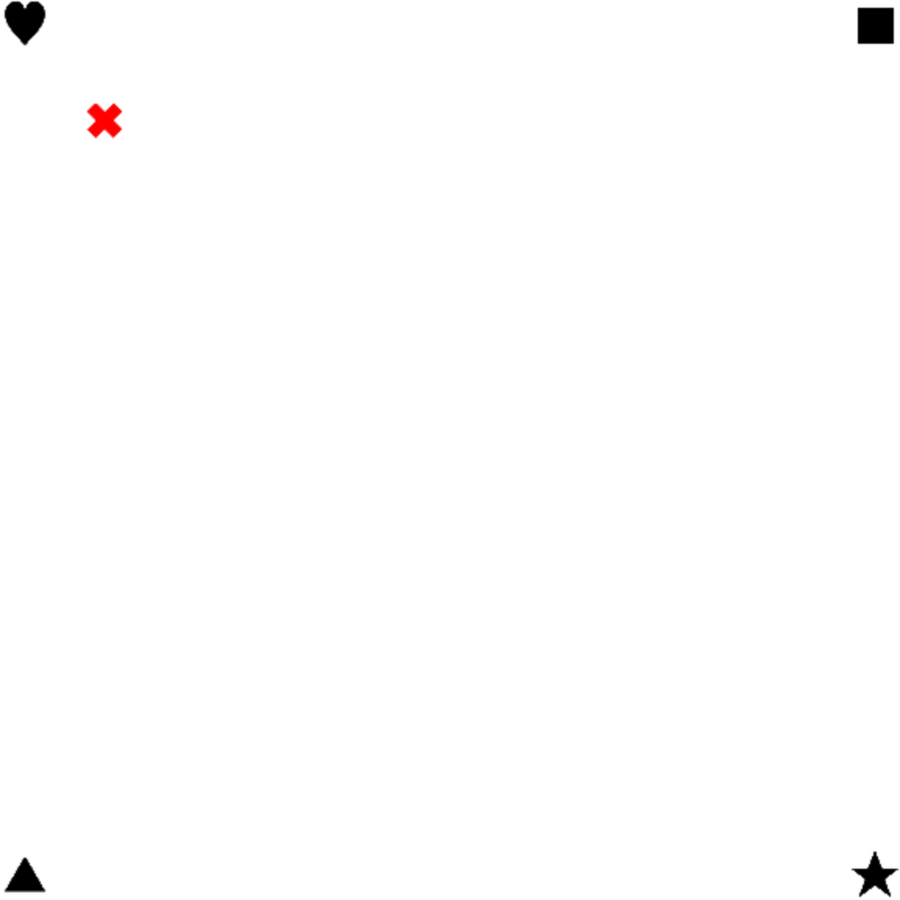}}
        %\caption*{(a)}
    \end{minipage}
    \hspace{5mm}
    \begin{minipage}{0.2\textwidth}
        \centering
        \fbox{\includegraphics[width=\linewidth]{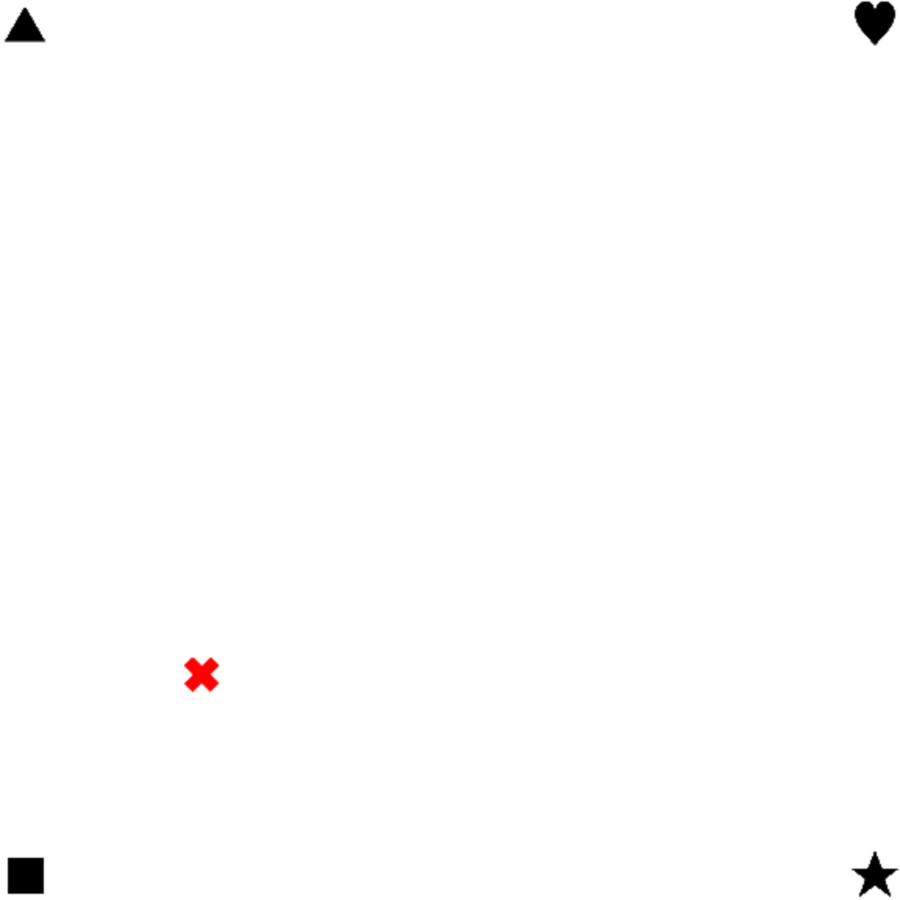}}
        %\caption*{(c)}
    \end{minipage}
    \hspace{5mm}
    \begin{minipage}{0.2\textwidth}
        \centering
        \fbox{\includegraphics[width=\linewidth]{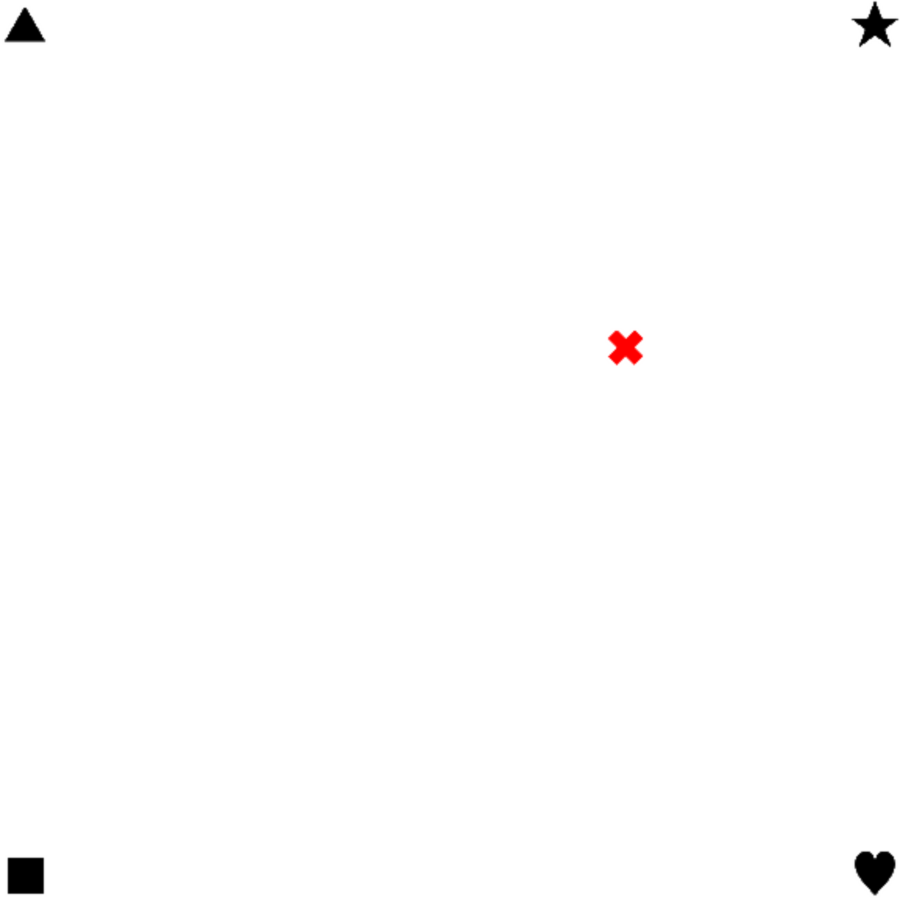}}
        %\caption*{(d)}
    \end{minipage}
    \caption{Examples of image inputs for the spatial resolution task.}
    \label{fig:wurstelony}
\end{figure}

\paragraph{Experimental setup.}
We showed each image to both Qwen and Gemma, together with a query prompt: \texttt{"The picture contains four black shapes: a square, triangle, heart and a star. There is also a red X. Which black shape is the closest to the X? Respond with only the shape's name."}. We logged the models' textual responses and the token scores for {\em square}, {\em triangle}, {\em heart} and a {\em star}. 
For each sampled location, we recorded the model's output and computed an accuracy map. A smoothed version of this map was obtained by averaging over local neighbourhoods to reveal confidence gradients.

\paragraph{Results.} 
Figure~\ref{fig:experiments}c (main text) shows the accuracy maps for both models. In both cases, we observe a central region around each shape where the model is consistently correct, surrounded by transition zones where accuracy rapidly deteriorates. This behavior mirrors the emergence of a finite resolution scale: when the red cross is placed sufficiently far from all reference shapes, the models are increasingly unable to resolve which object is closest.

Moreover, the spatial structure of the confusion regions reveals differences between models: \texttt{gemma-3-12b-it} exhibits a tighter high-confidence core, while \texttt{Qwen2.5-VL-7B-Instruct} shows broader transition bands, suggesting differences in their spatial encoding fidelity. The smoothed accuracy maps further support the hypothesis that VLMs implement a distance-dependent proximity function with finite support, analogous to the semantic similarity functions described in the theoretical model (Section~\ref{section:setup}).

\paragraph{Color similarity task}\label{sub:color} 

\paragraph{Dataset.}

We created 5,000 images, each containing between 4 and 12 colored squares. Each square was labeled with a unique letter, serving as an identifier, as shown in Figure~\ref{fig:pantony}. The colors of the squares were generated using the HSV color model, where the hue (H) was assigned randomly, while saturation (S) and value (V) were maximized to ensure vivid and bright colors.

\begin{figure}[h!]
    \centering
    \begin{minipage}{0.16\textwidth}
        \centering
        \fbox{\includegraphics[width=\linewidth]{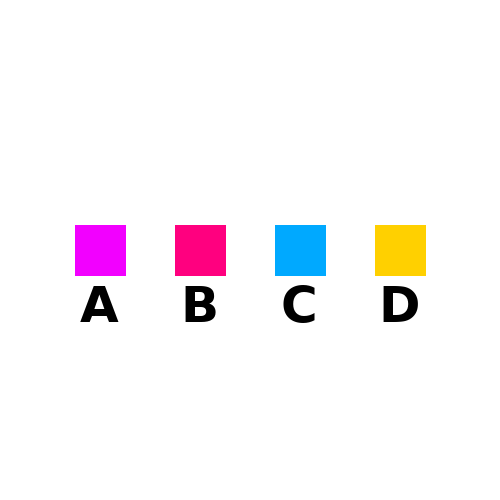}}
        \caption*{(a)}
    \end{minipage}
    \hspace{2mm}
    \begin{minipage}{0.16\textwidth}
        \centering
        \fbox{\includegraphics[width=\linewidth]{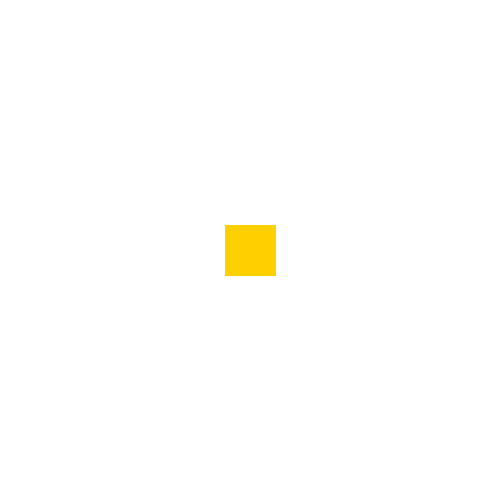}}
        \caption*{(b)}
    \end{minipage}
    \hspace{2mm}
    \begin{minipage}{0.16\textwidth}
        \centering
        \fbox{\includegraphics[width=\linewidth]{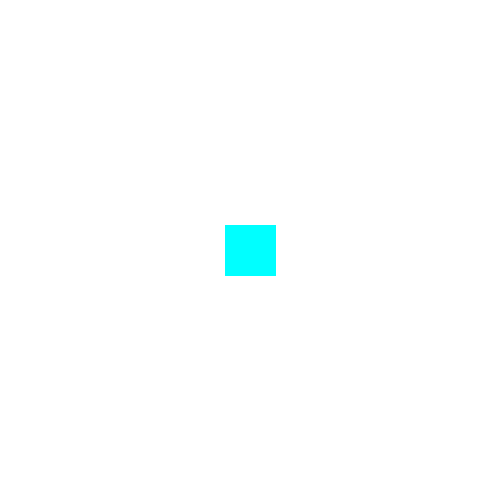}}
        \caption*{(c)}
    \end{minipage}
    \hspace{2mm}
    \begin{minipage}{0.16\textwidth}
        \centering
        \fbox{\includegraphics[width=\linewidth]{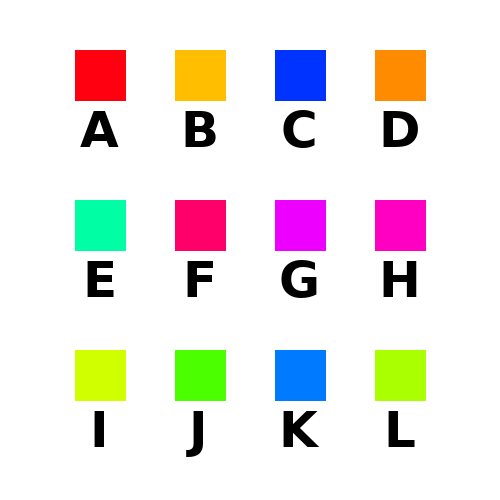}}
        \caption*{(d)}
    \end{minipage}
    \caption{Examples of image inputs for the color similarity task. Panels (a) and (d) represent two reference images, with four and twelve colors respectively.  
    Given the reference image (a), a query image for the identification task (color occurring in the reference image) is depicted in panel (b), and a query image for the similarity task (color not occurring in the reference image) is depicted in panel (c).}
    \label{fig:pantony}
\end{figure}

\paragraph{Experimental setup.}

In this task, we presented the models with a pair of images and a textual query. The first image, dubbed {\em reference} image, contained 4 to 12 labelled color squares, as described above. The second image ({\em query image}) displayed a single, centered square, whose color was either one of the colors occurring in the reference image (identification task) or a completely random one (similarity task). 
In both cases, the query was: \texttt{"In the first picture there are squares of different colors, labelled with uppercase letters. In the second picture there is only one target square. Identify which square in the first picture is most similar to the target square in the second picture. Reply with the corresponding letter and nothing else."}. 
We logged the color of the target square and its similarity with respect to all colors occurring in the paired reference image. We also logged the model's textual answers and the token score for each single letter (possible answers).

\paragraph{Results.} 
In \Cref{fig:pantony_resolution}a we show the similarity task and identification task performances as functions of the number of input colored squares.
In particular, we observe a decreasing identification performance which, in both models, can be fitted using the theoretical curve of \Cref{theo:n-item} (main text).
The fitted parameter $b(\varepsilon)$ suggests the presence of a larger effective resolution for Gemma and a lower one for Qwen.

To investigate this resolution, we gather, for each experiment, the scores each model assigns to the letters associated to the \emph{wrong} colors (thus excluding the most similar ones), together with their circular hue distance from the probe color, normalized to $[0,0.5]$.
We only take the scores associated to the wrong colors in order to avoid the bias of the correct answer having always low distance.

We plot the model score as a function of the distance in \Cref{fig:pantony_resolution}b.
Both model display an emergent resolution, with points with large hue distance being concentrated around a fixed \virg{noise} level.
Moreover, the scores for Gemma display a step-like shape suggesting that the learned similarity function may be similar to the constant similarity assumed in the theoretical analysis.
Qwen, instead, shows a more continuous decrease in score-similarity with distance, more in line with the results obtained in \Cref{sub:res_llm}, and associated to higher performances (\Cref{fig:pantony_resolution}a).

\begin{figure}
    \centering
    \includegraphics[width=\linewidth]{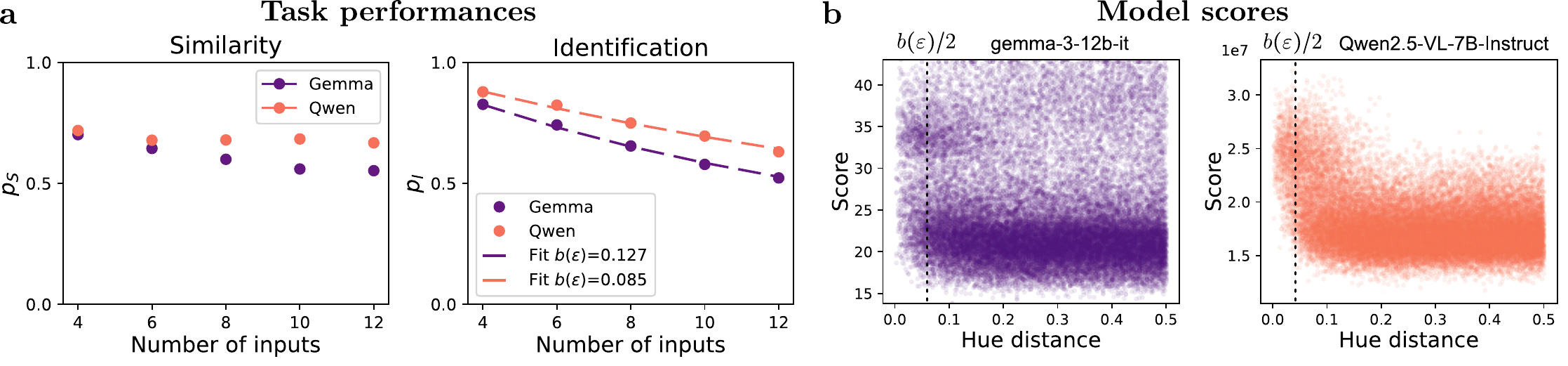}
    \caption{
    \textbf{a.} Similarity and identification probabilities for the color test explained in \Cref{sub:color}. In the identification plot, the dashed curves are the theoretical curves of \Cref{theo:n-item} (main text) fitted to the data. 
    \textbf{b.} Token score associated to the \emph{wrong} responses, as a function of the hue distance from the probe color. The dotted black lines represent the resolution values $b(\varepsilon)/2$ fitted using the theoretical resul of \Cref{theo:n-item} (main text) on the identification performances.}
    \label{fig:pantony_resolution}
\end{figure}

\putbib[biblio]
\end{bibunit}


\begin{thebibliography}{45}
\providecommand{\natexlab}[1]{#1}
\providecommand{\url}[1]{\texttt{#1}}
\expandafter\ifx\csname urlstyle\endcsname\relax
  \providecommand{\doi}[1]{doi: #1}\else
  \providecommand{\doi}{doi: \begingroup \urlstyle{rm}\Url}\fi

\bibitem[Campbell et~al.(2024)Campbell, Rane, Giallanza, De~Sabbata, Ghods, Joshi, Ku, Frankland, Griffiths, Cohen, et~al.]{campbell2024understanding}
Declan Campbell, Sunayana Rane, Tyler Giallanza, Camillo~Nicol{\`o} De~Sabbata, Kia Ghods, Amogh Joshi, Alexander Ku, Steven Frankland, Tom Griffiths, Jonathan~D Cohen, et~al.
\newblock Understanding the limits of vision language models through the lens of the binding problem.
\newblock \emph{Advances in Neural Information Processing Systems}, 37:\penalty0 113436--113460, 2024.

\bibitem[Hinton et~al.(1986)Hinton, McClelland, and Rumelhart]{Hinton}
G.~E. Hinton, J.~L. McClelland, and D.~E. Rumelhart.
\newblock \emph{Parallel distributed processing: Explorations in the microstructure of cognition}.
\newblock MIT Press, 1986.

\bibitem[Hinton(1986)]{Hinton86}
Geoffrey~E. Hinton.
\newblock Learning distributed representations of concepts.
\newblock \emph{In Proceedings of Eighth Annual Conference of the Cognitive Science Society}, 1986.
\newblock URL \url{https://www.cs.toronto.edu/~hinton/absps/families.pdf}.

\bibitem[Smolensky(1990)]{smolensky1990tensor}
Paul Smolensky.
\newblock Tensor product variable binding and the representation of symbolic structures in connectionist systems.
\newblock \emph{Artificial Intelligence}, 46\penalty0 (1--2):\penalty0 159--216, 1990.
\newblock \doi{10.1016/0004-3702(90)90007-M}.

\bibitem[Roskies(1999)]{Roskies99}
Adina~L. Roskies.
\newblock The binding problem.
\newblock \emph{Neuron}, 24, 1999.
\newblock URL \url{https://www.cell.com/neuron/fulltext/S0896-6273(00)80817-X?_returnURL=https%3A%2F%2Flinkinghub.elsevier.com%2Fretrieve%2Fpii%2FS089662730080817X%3Fshowall%3Dtrue}.

\bibitem[Greff et~al.(2020)Greff, van Steenkiste, and Schmidhuber]{Greff20}
Klaus Greff, Sjoerd van Steenkiste, and Jürgen Schmidhuber.
\newblock On the binding problem in artificial neural networks.
\newblock \emph{ArXiv}, 2020.
\newblock URL \url{https://arxiv.org/pdf/2012.05208}.

\bibitem[Treisman and Gelade(1980)]{treisman1980feature}
Anne~M Treisman and Garry Gelade.
\newblock A feature-integration theory of attention.
\newblock \emph{Cognitive psychology}, 12\penalty0 (1):\penalty0 97--136, 1980.

\bibitem[Shepard(1958)]{Shepard58}
Roger Shepard.
\newblock Stimulus and response generalization: Deduction of the generalization gradient from a trace model.
\newblock \emph{Psychological Review}, 1958.
\newblock URL \url{https://psycnet.apa.org/record/1959-09346-001}.

\bibitem[Shepard(1987)]{shepard1987toward}
Roger~N Shepard.
\newblock Toward a universal law of generalization for psychological science.
\newblock \emph{Science}, 237\penalty0 (4820):\penalty0 1317--1323, 1987.

\bibitem[Xie et~al.(2024)Xie, Ma, Thakurdesai, Kim, and Zhang]{Xie24}
Weizhen Xie, Tianye Ma, Sanikaa Thakurdesai, Inik Kim, and Weiwei Zhang.
\newblock Discrimination of mnemonic similarity is associated with short-term and long-term memory precision.
\newblock \emph{Memory \& Cognition}, pages 1--13, 2024.

\bibitem[Schurgin et~al.(2020)Schurgin, Wixted, and Brady]{Schurgin20}
Mark~W. Schurgin, John~T. Wixted, and Timothy~F. Brady.
\newblock Psychophysical scaling reveals a unified theory of visual memory strength.
\newblock \emph{Nature Human Behaviour}, 2020.
\newblock URL \url{https://www.nature.com/articles/s41562-020-00938-0}.

\bibitem[Tomic and Bays(2024)]{Tomic24}
Ivan Tomic and Paul~M. Bays.
\newblock Perceptual similarity judgments do not predict the distribution of errors in working memory.
\newblock \emph{Journal of Experimental Psychology: Learning, Memory, and Cognition}, 50, 2024.
\newblock URL \url{https://psycnet.apa.org/doiLanding?doi=10.1037%2Fxlm0001172}.

\bibitem[Zaslavsky et~al.(2018)Zaslavsky, Kemp, Regier, and Tishby]{zaslavsky2018efficient}
Noga Zaslavsky, Charles Kemp, Terry Regier, and Naftali Tishby.
\newblock Efficient compression in color naming and its evolution.
\newblock \emph{Proceedings of the National Academy of Sciences}, 115\penalty0 (31):\penalty0 7937--7942, 2018.

\bibitem[Cheng(2000)]{cheng2000honeybees}
Ken Cheng.
\newblock Shepard's universal law supported by honeybees in spatial generalization.
\newblock \emph{Psychological Science}, 11\penalty0 (5):\penalty0 403--408, 2000.
\newblock \doi{10.1111/1467-9280.00278}.
\newblock URL \url{https://doi.org/10.1111/1467-9280.00278}.

\bibitem[Hebart et~al.(2020)Hebart, Zheng, Pereira, and Baker]{hebart2020multidim}
Martin~N. Hebart, Charles~Y. Zheng, Francisco Pereira, and Chris~I. Baker.
\newblock Revealing the multidimensional mental representations of natural objects underlying human similarity judgements.
\newblock \emph{Nature Human Behaviour}, 4, 2020.
\newblock URL \url{https://www.nature.com/articles/s41562-020-00951-3#citeas}.

\bibitem[Tenenbaum and Griffiths(2001)]{Tenenbaum2001Generalization}
Joshua~B. Tenenbaum and Thomas~L. Griffiths.
\newblock Generalization, similarity, and {Bayesian} inference.
\newblock \emph{Behavioral and Brain Sciences}, 24\penalty0 (4):\penalty0 629--640, August 2001.
\newblock \doi{10.1017/s0140525x01000061}.

\bibitem[Sims(2018)]{Sims18}
Chris~R. Sims.
\newblock Efficient coding explains the universal law of generalization in human perception.
\newblock \emph{Science}, 360, 2018.
\newblock URL \url{https://www.science.org/doi/10.1126/science.aaq1118}.

\bibitem[Tishby and Zaslavsky(2015)]{tishby2015deep}
Naftali Tishby and Noga Zaslavsky.
\newblock Deep learning and the information bottleneck principle.
\newblock In \emph{2015 ieee information theory workshop (itw)}, pages 1--5. Ieee, 2015.

\bibitem[Sims(2016)]{Sims16}
Chris~R. Sims.
\newblock Rate-distortion theory and human perception.
\newblock \emph{Cognition}, 152, 2016.
\newblock URL \url{https://www.sciencedirect.com/science/article/pii/S0010027716300750}.

\bibitem[Tishby et~al.(1999)Tishby, Pereira, and Bialek]{Tishby99}
Naftali Tishby, Fernando~C. Pereira, and William Bialek.
\newblock The information bottleneck method.
\newblock \emph{ArXiv}, 1999.
\newblock URL \url{https://arxiv.org/pdf/physics/0004057}.

\bibitem[Schwartz-Ziv and Tishby(2017)]{Schwartz-Ziv17}
Ravid Schwartz-Ziv and Naftali Tishby.
\newblock Opening the black box of deep neural networks via information.
\newblock \emph{ArXiv}, 2017.
\newblock URL \url{https://arxiv.org/pdf/1703.00810}.

\bibitem[Gong and Zhang(2024)]{Gong24}
Dongyu Gong and Hantao Zhang.
\newblock Self-attention limits working memory capacity of transfomer-based models.
\newblock \emph{ArXiv}, 2024.
\newblock URL \url{https://arxiv.org/pdf/2409.10715}.

\bibitem[Frankland et~al.(2021)Frankland, Webb, Lewis, and Cohen]{frankland2021no}
Steven~M Frankland, Taylor Webb, Richard~L Lewis, and Jonathan~D Cohen.
\newblock No coincidence, george: Processing limits in cognitive function reflect the curse of generalization, Oct 2021.
\newblock URL \url{osf.io/preprints/psyarxiv/cjuxb_v1}.

\bibitem[Carlsson et~al.(2008)Carlsson, Ishkhanov, De~Silva, and Zomorodian]{carlsson2008local}
Gunnar Carlsson, Tigran Ishkhanov, Vin De~Silva, and Afra Zomorodian.
\newblock On the local behavior of spaces of natural images.
\newblock \emph{International journal of computer vision}, 76:\penalty0 1--12, 2008.

\bibitem[Vaswani et~al.(2017)Vaswani, Shazeer, Parmar, Uszkoreit, Jones, Gomez, Kaiser, and Polosukhin]{vaswani2017attention}
Ashish Vaswani, Noam Shazeer, Niki Parmar, Jakob Uszkoreit, Llion Jones, Aidan~N Gomez, {\L}ukasz Kaiser, and Illia Polosukhin.
\newblock Attention is all you need.
\newblock \emph{Advances in neural information processing systems}, 30, 2017.

\bibitem[Luce(1959)]{luce1959individual}
R~Duncan Luce.
\newblock \emph{Individual choice behavior}, volume~4.
\newblock Wiley New York, 1959.

\bibitem[Elhage et~al.(2022)Elhage, Hume, Olsson, Schiefer, Henighan, Kravec, Hatfield-Dodds, Lasenby, Drain, Chen, et~al.]{elhage2022toy}
Nelson Elhage, Tristan Hume, Catherine Olsson, Nicholas Schiefer, Tom Henighan, Shauna Kravec, Zac Hatfield-Dodds, Robert Lasenby, Dawn Drain, Carol Chen, et~al.
\newblock Toy models of superposition.
\newblock \emph{arXiv preprint arXiv:2209.10652}, 2022.

\bibitem[He et~al.(2016)He, Zhang, Ren, and Sun]{he2016deep}
Kaiming He, Xiangyu Zhang, Shaoqing Ren, and Jian Sun.
\newblock Deep residual learning for image recognition.
\newblock In \emph{Proceedings of the IEEE Conference on Computer Vision and Pattern Recognition}, pages 770--778, 2016.
\newblock \doi{10.1109/CVPR.2016.90}.

\bibitem[Wah et~al.(2011)Wah, Branson, Welinder, Perona, and Belongie]{wah2011caltech}
Catherine Wah, Steve Branson, Peter Welinder, Pietro Perona, and Serge Belongie.
\newblock The {Caltech-UCSD} {Birds-200-2011} dataset.
\newblock Technical Report CNS-TR-2011-001, California Institute of Technology, 2011.

\bibitem[Kumar et~al.(2022)Kumar, Suleski, Craig, Kasprowicz, Sanderford, Li, Stecher, and Hedges]{kumar2022timetree}
Sudhir Kumar, Morgan Suleski, Jessica~M Craig, Anna~E Kasprowicz, Maxwell Sanderford, Mingfeng Li, Glen Stecher, and S~Blair Hedges.
\newblock {TimeTree 5}: An expanded resource for species divergence times.
\newblock \emph{Molecular Biology and Evolution}, 39\penalty0 (8):\penalty0 msac174, 2022.
\newblock \doi{10.1093/molbev/msac174}.

\bibitem[Team et~al.(2024)Team, Riviere, Pathak, Sessa, Hardin, Bhupatiraju, Hussenot, Mesnard, Shahriari, Ram{\'e}, et~al.]{team2024gemma}
Gemma Team, Morgane Riviere, Shreya Pathak, Pier~Giuseppe Sessa, Cassidy Hardin, Surya Bhupatiraju, L{\'e}onard Hussenot, Thomas Mesnard, Bobak Shahriari, Alexandre Ram{\'e}, et~al.
\newblock Gemma 2: Improving open language models at a practical size.
\newblock \emph{arXiv preprint arXiv:2408.00118}, 2024.

\bibitem[Grattafiori et~al.(2024)Grattafiori, Dubey, Jauhri, Pandey, Kadian, Al-Dahle, Letman, Mathur, Schelten, Vaughan, et~al.]{grattafiori2024llama}
Aaron Grattafiori, Abhimanyu Dubey, Abhinav Jauhri, Abhinav Pandey, Abhishek Kadian, Ahmad Al-Dahle, Aiesha Letman, Akhil Mathur, Alan Schelten, Alex Vaughan, et~al.
\newblock The llama 3 herd of models.
\newblock \emph{arXiv preprint arXiv:2407.21783}, 2024.

\bibitem[Yang et~al.(2024)Yang, Yang, Zhang, Hui, Zheng, Yu, Li, Liu, Huang, Wei, et~al.]{yang2024qwen2}
An~Yang, Baosong Yang, Beichen Zhang, Binyuan Hui, Bo~Zheng, Bowen Yu, Chengyuan Li, Dayiheng Liu, Fei Huang, Haoran Wei, et~al.
\newblock Qwen2. 5 technical report.
\newblock \emph{arXiv preprint arXiv:2412.15115}, 2024.

\bibitem[Team(2025{\natexlab{a}})]{gemma_techrep}
Gemma Team.
\newblock Gemma 3.
\newblock 2025{\natexlab{a}}.
\newblock URL \url{https://goo.gle/Gemma3Report}.

\bibitem[Team(2025{\natexlab{b}})]{qwen_techrep}
Qwen Team.
\newblock Qwen2.5-vl, January 2025{\natexlab{b}}.
\newblock URL \url{https://qwenlm.github.io/blog/qwen2.5-vl/}.

\bibitem[Sorscher et~al.(2022)Sorscher, Ganguli, and Sompolinsky]{sorscher2022neural}
Ben Sorscher, Surya Ganguli, and Haim Sompolinsky.
\newblock Neural theory for few-shot learning of naturalistic stimuli.
\newblock \emph{Proceedings of the National Academy of Sciences}, 119\penalty0 (12):\penalty0 e2112410119, 2022.
\newblock \doi{10.1073/pnas.2112410119}.

\bibitem[Petri et~al.(2024)Petri, Musslick, and Cohen]{Petri24}
Giovanni Petri, Sebastian Musslick, and Jonathan~D. Cohen.
\newblock An information-theoretic approach to reward rate optimization in the tradeoff between controlled and automatic processing in neural network architectures.
\newblock \emph{eLife}, 13, 2024.
\newblock URL \url{https://elifesciences.org/reviewed-preprints/93251}.

\bibitem[Petri et~al.(2021)Petri, Musslick, Dey, {\"O}zcimder, Turner, Ahmed, Willke, and Cohen]{petri2021topological}
Giovanni Petri, Sebastian Musslick, Biswadip Dey, Kayhan {\"O}zcimder, David Turner, Nesreen~K Ahmed, Theodore~L Willke, and Jonathan~D Cohen.
\newblock Topological limits to the parallel processing capability of network architectures.
\newblock \emph{Nature Physics}, 17\penalty0 (5):\penalty0 646--651, 2021.

\bibitem[Lesnick et~al.(2020)Lesnick, Musslick, Dey, and Cohen]{Lesnick20}
Michael Lesnick, Sebastian Musslick, Biswadip Dey, and Jonathan~D. Cohen.
\newblock A formal framework for cognitive models of multitasking.
\newblock \emph{PsyArXiv}, 2020.
\newblock URL \url{https://osf.io/preprints/psyarxiv/7yzdn_v1}.

\bibitem[Cohen et~al.(2020)Cohen, Chung, Lee, and Sompolinsky]{cohen2020separability}
Uri Cohen, SueYeon Chung, Daniel~D Lee, and Haim Sompolinsky.
\newblock Separability and geometry of object manifolds in deep neural networks.
\newblock \emph{Nature Communications}, 11\penalty0 (1):\penalty0 746, 2020.
\newblock \doi{10.1038/s41467-020-14578-5}.

\bibitem[Ganmor et~al.(2015)Ganmor, Segev, and Schneidman]{Ganmor15}
Elad Ganmor, Ronen Segev, and Elad Schneidman.
\newblock A thesaurus for a neural population code.
\newblock \emph{eLife}, 2015.
\newblock URL \url{https://elifesciences.org/articles/06134.pdf}.

\bibitem[Curto et~al.(2013)Curto, Itskov, Morrison, Roth, and Walker]{curto2013combinatorial}
Carina Curto, Vladimir Itskov, Katherine Morrison, Zachary Roth, and Judy~L Walker.
\newblock Combinatorial neural codes from a mathematical coding theory perspective.
\newblock \emph{Neural computation}, 25\penalty0 (7):\penalty0 1891--1925, 2013.

\bibitem[Lake and Baroni(2023)]{LakeBaroni23}
Brenden~M. Lake and Marco Baroni.
\newblock Human-like systematic generalization through a meta-learning neural network.
\newblock \emph{Nature}, 623, 2023.
\newblock URL \url{https://www.nature.com/articles/s41586-023-06668-3}.

\bibitem[Fodor and Pylyshyn(1998)]{FodorPylyshyn98}
Jerry~A. Fodor and Zenon~W. Pylyshyn.
\newblock Connectionism and cognitive architecture: A critical analysis.
\newblock \emph{Cognition}, 28, 1998.
\newblock URL \url{https://www.sciencedirect.com/science/article/abs/pii/0010027788900315}.

\bibitem[Proca et~al.(2024)Proca, Rosas, Luppi, Bor, Crosby, and Mediano]{proca2024synergistic}
Alexandra~M Proca, Fernando~E Rosas, Andrea~I Luppi, Daniel Bor, Matthew Crosby, and Pedro~AM Mediano.
\newblock Synergistic information supports modality integration and flexible learning in neural networks solving multiple tasks.
\newblock \emph{PLoS computational biology}, 20\penalty0 (6):\penalty0 e1012178, 2024.

\end{thebibliography}


\begin{thebibliography}{9}
\providecommand{\natexlab}[1]{#1}
\providecommand{\url}[1]{\texttt{#1}}
\expandafter\ifx\csname urlstyle\endcsname\relax
  \providecommand{\doi}[1]{doi: #1}\else
  \providecommand{\doi}{doi: \begingroup \urlstyle{rm}\Url}\fi

\bibitem[Nielsen(1997)]{nielsen1997introduction}
Ole~A Nielsen.
\newblock An introduction to integration and measure theory.
\newblock \emph{(No Title)}, 1997.

\bibitem[Folland(1999)]{folland1999real}
Gerald~B Folland.
\newblock \emph{Real analysis: modern techniques and their applications}.
\newblock John Wiley \& Sons, 1999.

\bibitem[Graham et~al.(1989)Graham, Knuth, and Patashnik]{graham89_concrete}
Ronald~L. Graham, Donald~E. Knuth, and Oren Patashnik.
\newblock \emph{{Concrete Mathematics}}.
\newblock Addison-Wesley, 1989.
\newblock ISBN 0-201-14236-8.

\bibitem[He et~al.(2016)He, Zhang, Ren, and Sun]{he2016deep}
Kaiming He, Xiangyu Zhang, Shaoqing Ren, and Jian Sun.
\newblock Deep residual learning for image recognition.
\newblock In \emph{Proceedings of the IEEE Conference on Computer Vision and Pattern Recognition}, pages 770--778, 2016.
\newblock \doi{10.1109/CVPR.2016.90}.

\bibitem[Wah et~al.(2011)Wah, Branson, Welinder, Perona, and Belongie]{wah2011caltech}
Catherine Wah, Steve Branson, Peter Welinder, Pietro Perona, and Serge Belongie.
\newblock The {Caltech-UCSD} {Birds-200-2011} dataset.
\newblock Technical Report CNS-TR-2011-001, California Institute of Technology, 2011.

\bibitem[Kumar et~al.(2022)Kumar, Suleski, Craig, Kasprowicz, Sanderford, Li, Stecher, and Hedges]{kumar2022timetree}
Sudhir Kumar, Morgan Suleski, Jessica~M Craig, Anna~E Kasprowicz, Maxwell Sanderford, Mingfeng Li, Glen Stecher, and S~Blair Hedges.
\newblock {TimeTree 5}: An expanded resource for species divergence times.
\newblock \emph{Molecular Biology and Evolution}, 39\penalty0 (8):\penalty0 msac174, 2022.
\newblock \doi{10.1093/molbev/msac174}.

\bibitem[Team et~al.(2024)Team, Riviere, Pathak, Sessa, Hardin, Bhupatiraju, Hussenot, Mesnard, Shahriari, Ram{\'e}, et~al.]{team2024gemma}
Gemma Team, Morgane Riviere, Shreya Pathak, Pier~Giuseppe Sessa, Cassidy Hardin, Surya Bhupatiraju, L{\'e}onard Hussenot, Thomas Mesnard, Bobak Shahriari, Alexandre Ram{\'e}, et~al.
\newblock Gemma 2: Improving open language models at a practical size.
\newblock \emph{arXiv preprint arXiv:2408.00118}, 2024.

\bibitem[Grattafiori et~al.(2024)Grattafiori, Dubey, Jauhri, Pandey, Kadian, Al-Dahle, Letman, Mathur, Schelten, Vaughan, et~al.]{grattafiori2024llama}
Aaron Grattafiori, Abhimanyu Dubey, Abhinav Jauhri, Abhinav Pandey, Abhishek Kadian, Ahmad Al-Dahle, Aiesha Letman, Akhil Mathur, Alan Schelten, Alex Vaughan, et~al.
\newblock The llama 3 herd of models.
\newblock \emph{arXiv preprint arXiv:2407.21783}, 2024.

\bibitem[Bai et~al.(2023)Bai, Yang, Chai, Ling, Yang, Lei, Huang, Tan, Liu, Yang, et~al.]{bai2023qwen}
Jinze Bai, Yixuan Yang, Yingqi Chai, Victor Ling, Aohan Yang, Zhiyuan Lei, Junyang Huang, Yonggang Tan, Xiubo Liu, Zhijian Yang, et~al.
\newblock Qwen technical report.
\newblock \emph{arXiv preprint arXiv:2309.16609}, 2023.

\end{thebibliography}
\end{document}